\documentclass[10pt]{article}

\usepackage{amsmath,amsthm,amssymb, enumerate, breqn}
\usepackage{amsfonts}  
\usepackage{graphicx}
\usepackage{color}
\usepackage[margin=3cm]{geometry}
\usepackage{slashbox,placeins}

\usepackage{charter}
\usepackage{tikz}
\usepackage[utf8]{inputenc} 
\usepackage[T1]{fontenc}    
\usepackage{hyperref}       
\usepackage{url}            
\usepackage{booktabs}       
\usepackage{amsfonts}       
\usepackage{nicefrac}       
\usepackage{microtype}      
\usepackage{fullpage}
\usepackage{authblk}

\newcommand{\removed}[1]{} 
\newcommand{\bb}{\mathbb}
\newcommand{\R}{\bb R}

\theoremstyle{definition}
\newtheorem{theorem}{Theorem}[section]
\newtheorem{lemma}[theorem]{Lemma}

\def\ve#1{\mathchoice{\mbox{\boldmath$\displaystyle\bf#1$}}
{\mbox{\boldmath$\textstyle\bf#1$}}
{\mbox{\boldmath$\scriptstyle\bf#1$}}
{\mbox{\boldmath$\scriptscriptstyle\bf#1$}}}

\newcommand{\x}{{\ve x}}
\newcommand{\y}{{\ve y}}

\newif\ifsolutions \solutionstrue

\title{Sparse coding and autoencoders}
\author[1]{Akshay Rangamani{\bf \thanks{Equal Contribution}} \thanks{\url{rangamani.akshay@jhu.edu}}}
\author[2]{Anirbit Mukherjee${\bf ^*}$\thanks{\url{amukhe14@jhu.edu}}}
\author[2]{\protect \\Amitabh Basu \thanks{\url{basu.amitabh@jhu.edu}}}
\author[4]{Tejaswini Ganapathy \thanks{\url{tganapathi@salesforce.com}}}
\author[1]{\protect \\Ashish Arora \thanks{\url{aarora8@jhu.edu}}}
\author[1,3]{Sang (Peter) Chin \thanks{\url{spchin@cs.bu.edu}}}
\author[1]{Trac D. Tran \thanks{\url{trac@jhu.edu}}}

\affil[1]{ECE Department\\
Johns Hopkins University}
\affil[2]{AMS Department\\
Johns Hopkins University}
\affil[3]{Department of Computer Science\\ Boston University}
\affil[4]{Salesforce, San Francisco Bay Area}
\date{}

\begin{document}
\maketitle

\begin{abstract}
~\\
In \emph{Dictionary Learning} one tries to recover incoherent matrices $A^* \in \mathbb{R}^{n \times h}$ (typically overcomplete and whose columns are assumed to be normalized) and sparse vectors $x^* \in \mathbb{R}^h$ with a small support of size $h^p$ for some $0 <p < 1$ while having access to observations $y \in \mathbb{R}^n$ where $y = A^*x^*$. In this work we undertake a rigorous analysis of whether gradient descent on the squared loss of an autoencoder can solve the dictionary learning problem. The \emph{Autoencoder} architecture we consider is a $\mathbb{R}^n \rightarrow \mathbb{R}^n$ mapping with a single ReLU activation layer of size $h$. 
~\\ \\
Under very mild distributional assumptions on $x^*$, we prove that the norm of the expected gradient of the standard squared loss function is asymptotically (in sparse code dimension) negligible for {\em all} points in a small neighborhood of $A^*$. This is supported with experimental evidence using synthetic data. We also conduct experiments to suggest that $A^*$ is a local minimum. Along the way we prove that a layer of ReLU gates can be set up to automatically recover the support of the sparse codes. This property holds independent of the loss function. We believe that it could be of independent interest.
\end{abstract}

\section{Introduction}
One of the fundamental themes in learning theory is to consider data being sampled from a generative model and to provide efficient methods to recover the original model parameters exactly or with tight approximation guarantees. Classic examples include learning a mixture of gaussians \cite{moitra2010settling}, certain graphical models \cite{anandkumar2014tensor}, full rank square dictionaries \cite{spielman2012exact,blasiok2016improved} and overcomplete dictionaries \cite{agarwal2014learning,arora2014more,arora2015simple, arora2014new} The problem is usually distilled down to a non-convex optimization problem whose solution can be used to obtain the model parameters. With these hard non-convex problems it has been difficult to find any universal view as to why sometimes gradient descent gives very good and sometimes even exact recovery. In recent times progress has been made towards achieving a geometric understanding of the landscape of such non-convex optimization problems \cite{ge2017no}, \cite{mei2016landscape}, \cite{wu2017towards}. The corresponding question of parameter recovery for neural nets with one layer of activation has been solved in some special cases, \cite{du2017convolutional, allen2017natasha2, janzamin2015beating,sedghi2014provable, li2017convergence, tian2017analytical,zhang2017electron}. Almost all of these cases are in the supervised setting where it has also been assumed that the labels are being generated from a net of the same architecture as is being trained. In contrast to these works we address an unsupervised learning problem, and possibly more realistically, we do not tie the data generation model (sensing of sparse vectors by an overcomplete incoherent dictionary) to the neural architecture being analyzed except for assuming knowledge of a few parameters about the ground truth. In a related development, it has been shown by two of the authors here in a previous work~\cite{arora2016understanding}, that for two layer deep nets even the exact global minima can be found deterministically in time polynomial in the data size. This work continues that line of investigation to now make use of generative model assumptions to probe the power of a class of two layer deep nets with ReLU activation. 
~\\ \\ 
Here we specialize to the generative model of  {\it dictionary learning/sparse coding} where one receives samples of vectors $y \in \mathbb{R}^n$ that have been generated as $y = A^*x^*$ where $A^* \in \mathbb{R}^{n \times h}$ and $x^* \in \mathbb{R}^h$. We typically assume that the number of non-zero entries in $x^*$ to be no larger than some function of the dimension $h$ and that $A^*$ satisfies certain incoherence properties. The question now is to recover $A^*$ from samples of $y$. There have been renewed investigations into the hardness of this problem \cite{tillmann2015computational} and many former results have recently been reviewed in these lectures \cite{GilbertLectures}. This question has been a cornerstone of learning theory ever since the ground-breaking paper by Olshausen and Field (\cite{olshausen1997sparse}) (a recent review by the same authors can be found in \cite{olshausen2005close}). Over the years many algorithms have been developed to solve this problem and a detailed comparison among these various approaches can be found in \cite{blasiok2016improved}.
~\\ \\
An {\em autoencoder} is a neural network that maps $\mathbb{R}^n \rightarrow \mathbb{R}^n$ with a single hidden layer of Rectified Linear Unit (ReLU) activations. These networks have been used extensively (\cite{baldi2012autoencoders,bengio2013generalized, rifai2011contractive, vincent2008extracting,   vincent2010stacked}) in the past for unsupervised feature learning tasks, and have been found to be successful in generating discriminative features \cite{coates2011analysis}. A number of different autoencoder architectures and regularizers have been proposed which purportedly induce sparsity, at the hidden layer \cite{arpit2016regularized,coates2011importance,li2016sparseness, ng2011sparse}. There has also been some investigation into what autoencoders learn about the data distribution \cite{alain2014regularized}.
~\\ \\
Olshausen and Field had, as early as $1996$, already made the connection between sparse coding and training neural architectures and in today's terminology this problem is very naturally reminiscent of the architecture of an autoencoder \cite{olshausen1996emergence}. However, to the best of our knowledge, there has not been sufficient progress to rigorously establish whether autoencoders can do sparse coding. 
In this work, we present our progress towards bridging the above mentioned mathematical gap. To the best of our knowledge, there is no theoretical evidence (even under the usual generative assumptions of sparse coding) that the stationary points of any of the usual squared loss functions (with or without any of the usual regularizers) have any resemblance to the original dictionary that is being sought to be learned. {\bf The main point of this paper is to rigorously prove that for autoencoders with ReLU activation, the standard squared loss function has a neighborhood around the dictionary $A^*$ where the norm of the expected gradient is very small (for large enough sparse code dimension $h$). Thus, all points in a neighborhood of $A^*$, including $A^*$, are all asymptotic critical points of this standard squared loss.} 
We supplement our theoretical result with experimental evidence for it in Section~\ref{sec:experiments}, which also strongly suggests that the standard squared loss function has a local minimum in a neighborhood around $A^*$. We believe that our results provide theoretical and experimental evidence that the sparse coding problem can be tackled by training autoencoders.

\newpage 
\subsection {A motivating experiment on MNIST using TensorFlow}

We used TensorFlow \cite{abadi2016tensorflow} to train two ReLU autoencoders mapping $\R^{784} \rightarrow \R^{784}$. These networks were trained on a subset of the MNIST dataset of handwritten digits. One of the nets had a single hidden layer of size $10000$ and the other one had two hidden layers of size $5000$ and $784$ (and a fixed identity matrix giving the output from the second layer of activations). In both the cases the weights of the encoder and decoder were maintained as transposes of each other. We trained the autoencoders on the standard squared loss function using RMSProp \cite{tieleman2015rmsprop}. The training was done on $6000$ images of the digits $6$ and $7$ from the MNIST dataset. In the following panel we show four pairs (two for each net) of ``reconstructed" image i.e output of the trained net when its given as input the ``actual" photograph as input. 

\begin{figure}[h] 
\centering
\includegraphics[scale=0.30]{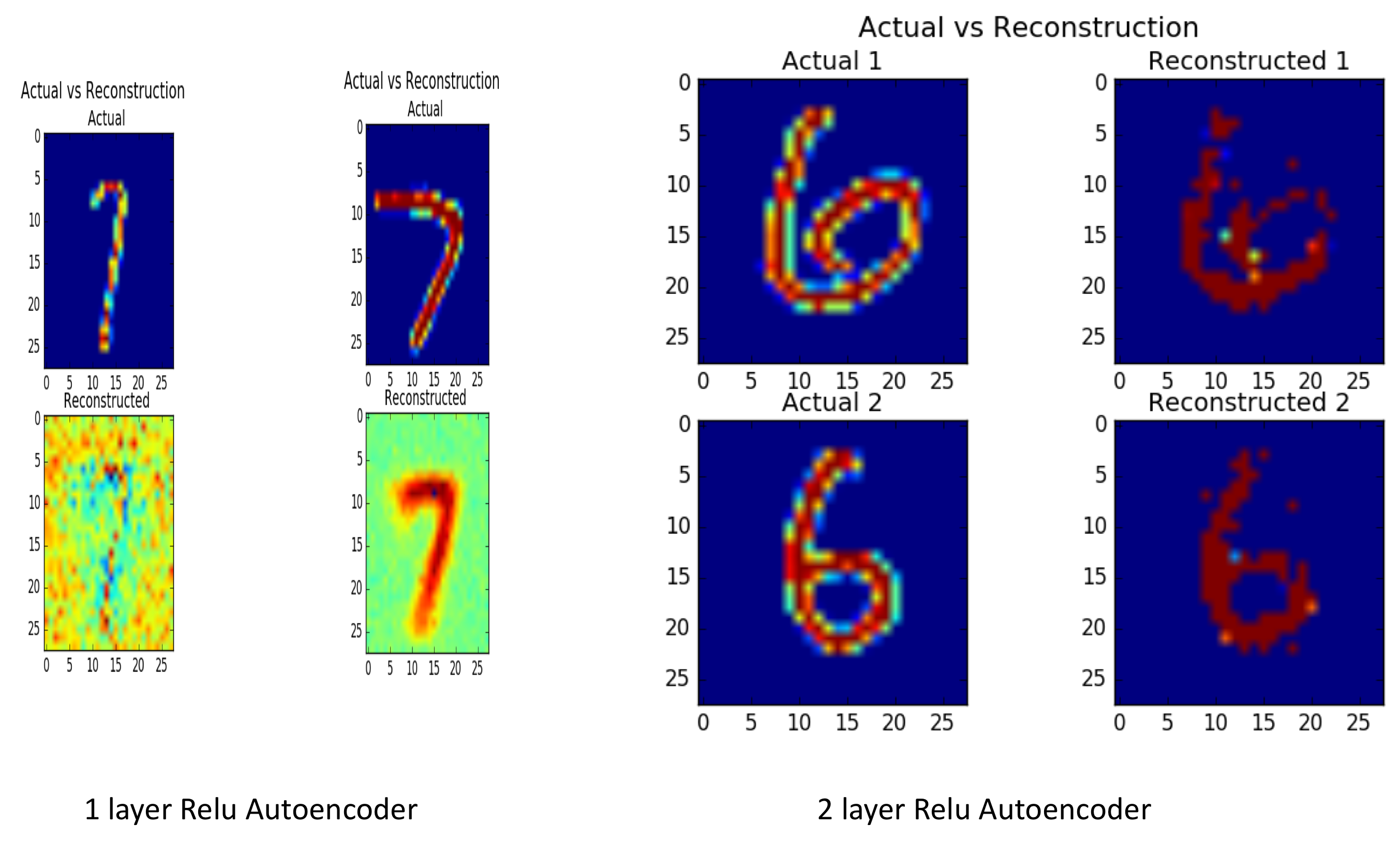}
\end{figure}
~\\
In our opinion, the above figures add support to the belief that a single and a double layer ReLU activated $\R^n \rightarrow \R^n$ network can learn an implicit high dimensional structure about the handwritten digits dataset. In particular this demonstrates that though adding more hidden layers obviously helps enhance the reconstruction ability, the single hidden layer autoencoder do hold within them significant power for unsupervised learning of representations. Unfortunately analyzing the RMSProp update rule used in the above experiment is currently beyond our analytic means. However, we take inspiration from these experiments to devise a different mathematical set-up which is much more amenable to analysis taking us towards a better understanding of the power of autoencoders.

\newpage 
\section{Introducing the neural architecture and the distributional assumptions}\label{sec:defn}

For any $n,h \in \{1,2,..\}$, an autoencoder is a fully connected $\mathbb{R}^n \rightarrow \mathbb{R}^n$ neural network with a single hidden layer of $h$ activations. We focus on networks that use the Rectified Linear Unit (ReLU) activation which is the function $\textrm{ReLU} : \mathbb{R}^h \rightarrow \mathbb{R}^h$ mapping $\vec{x} \rightarrow (\max\{0,x_i\})_{i=1}^{h}$. In this case, the autoencoder can be seen as computing the following function $\hat{y}(W,y,\epsilon)$ as follows,
\begin{align}
\label{eqn:autoencoder}
\nonumber r &= \textrm{ReLU}\left( Wy - \epsilon \right) \\
\hat{y} &= W^\top r
\end{align}
Here $y \in \mathbb{R}^n$ is the input to the autoencoder, $W \in \mathbb{R}^{h \times n}$ is the linear transformation implemented by the first layer, $r \in \mathbb{R}^h$ is the output of the layer of activations, $\epsilon \in \mathbb{R}^{h}$ is the bias vector and $\hat{y} \in \mathbb{R}^n$ is the output of the autoencoder. Note that we impose the condition that the second layer of weights is simply the transpose of the first layer. We shall define the columns of $W^\top$ (rows of $W$) as $\{W_i\}_{i=1}^h$. 

\paragraph{Assumptions on the dictionary and the sparse code.}\label{assumptions} We assume that our signal $y$ is generated using sparse linear combinations of atoms/vectors of an overcomplete dictionary, i.e., $y = A^* x^*$, where $A^* \in \mathbb{R}^{n \times h}$ is a dictionary, and $x^* \in (\mathbb{R}^{\geq 0})^h$ is a non-negative sparse vector, with at most $k=h^p$ (for some $0 < p < 1$) non zero elements. The columns of the original dictionary $A^*$ (labeled as $\{ A^*_i \}_{i=1}^{h}$) are assumed to be normalized and also satisfy the incoherence property that  $\max_{\substack{i,j=1,..,h \\ i \neq j}} \vert \langle A^*_i, A^*_j \rangle \vert \leq \frac{\mu}{\sqrt{n}} = h^{-\xi}$ for some $\xi >0$.  
\newline \newline 
We assume that the sparse code $x^*$ is sampled from a distribution with the following properties. We fix a set of possible supports of $x^*$, denoted by $\mathbb{S} \subseteq 2^{[h]}$, where each element of $\mathbb{S}$ has at most $k = h^p$ elements. We consider any arbitrary discrete probability distribution $D_{\mathbb{S}}$ on $\mathbb{S}$ such that the probability $q_1 := \mathbb{P}_{S \sim \mathbb{S}}[i \in S]$ is independent of $i \in [h]$, and the probability $q_{2}  := \mathbb{P}_{S \in \mathbb{S}}[i,j \in S]$ is independent of $i,j \in [h]$. A special case is when $\mathbb{S}$ is the set of all subsets of size $k$, and $D_{\mathbb{S}}$ is the uniform distribution on $\mathbb{S}$. For every $S \in \mathbb{S}$ there is a distribution say $D_S$ on $(\mathbb{R}^{\geq 0})^{h}$ which is supported on vectors whose support is contained in $S$ and which is uncorrelated for pairs of coordinates $i,j \in S$. Further, we assume that the distributions $D_S$ are such that each coordinate $i$ is compactly supported over an interval $[a(h),b(h)]$, where $a(h)$ and $b(h)$ are independent of both $i$ and $S$ but will be functions of $h$. Moreover, $m_1(h) := \mathbb{E}_{x^*\sim D_S}[x_i^*]$, and $m_2(h) :=\mathbb{E}_{x^*\sim D_S}[x_i^{*2}]$ are assumed to be independent of both $i$ and $S$ but allowed to depend on $h$. For ease of notation henceforth we will keep the $h$ dependence of these variables implicit and refer to them as $a, b, m_1$ and $m_2$. All of our results will hold in the special case when $a,b, m_1, m_2$ are constants (no dependence on $h$).  

\newpage 
\section{Main Results}
 
\subsection{Recovery of the support of the sparse code by a layer of ReLUs}\label{sec:results-support}

First we prove the following theorem which precisely quantifies the sense in which a layer of ReLU gates is able to recover the support of the sparse code when the weight matrix of the deep net is close to the original dictionary. We recall that the size of the support of the sparse vector $x^*$ is $k=h^p$ for some $0 < p < 1$. We also recall the parameters $a,b$ as defining the support of the marginal distribution of each coordinate of $x^*$ and $m_1$ is the expected value of this marginal distribution (recall that none of these depend on the coordinate or the actual support). These parameters will be referenced in the results below.

\begin{theorem}\label{sec:theorem:support}
 
Let each column of $W^\top$ be within a $\delta$-ball of the corresponding column of $A^*$, where $\delta = O \left( h^{-p - \nu^2} \right)$ for some $\nu>0$, such that $p + \nu^2 < \xi$ (where $h^{-\xi}$ is the coherence parameter). We further assume that $a = \Omega \left( b h^{ - \nu^2} \right)$. Let the bias of the hidden layer of the autoencoder, as defined in \eqref{eqn:autoencoder} be $\epsilon = 2 m_1 k \left( \delta + \frac{\mu}{\sqrt{n}} \right)$. Let $r$ be the vector defined in~\eqref{eqn:autoencoder}. Then $r_i \neq 0$ if $i \in \textrm{supp}(x^*)$, and $r_i =0$ if $i \notin \textrm{supp}(x^*)$ with probability at least $1 - \exp \left ( -\frac{2 h^p m_1^2}{(b-a)^2} \right )$ (with respect to the distribution on $x^*$). \end{theorem}
~\\ As long as $\frac{h^p m_1^2}{(b-a)^2}$ is large, i.e., an increasing function of $h$, we can interpret this as saying that the probability of the adverse event is small, and we have successfully achieved support recovery at the hidden layer in the limit of large sparse code dimension.

\subsection{Asymptotic Criticality of the Autoencoder around $A^*$}\label{sec:result-loss}

In this work we analyze the following standard squared loss function for the autoencoder, 
\begin{align}\label{eqn:loss}
L = \frac{1}{2} \vert \vert \hat{y} - y \vert \vert^2
\end{align}
In the above we continue to use the variables as defined in equation \ref{eqn:autoencoder}. If we consider a generative model in which $A^*$ is a square, orthogonal matrix and $x^*$ is a non-negative vector (not necessarily sparse), it is easily seen that the standard squared reconstruction error loss function for the autoencorder has a global minimum at $W = A^{*\top}$. In our generative model, however, $A^*$ is an incoherent and overcomplete dictionary.

\begin{theorem}\label{sec:theorem:critical} ({\bf The Main Theorem})
Assume that the hypotheses of Theorem \ref{theorem:support} hold, and $p < \min \{ \frac{1}{2}, \nu^2\}$ (and hence $\xi > 2p$). Further, assume the distribution parameters satisfy $\textrm{exp} \left(\frac{h^p m_1^2}{2(b-a)^2}\right)$ is superpolynomial in $h$ (which holds, for example, when $m_1, a, b$ are $O(1)$). Then for $i=1, \ldots, h$, 
 $$
 \bigg\Vert \mathbb{E}\left [ \frac {\partial L}{\partial W_i}\right ] \bigg\Vert_2 \leq o\bigg(\frac{\max\{m_1^2, m_2\}}{h^{1-p}}\bigg). 
 $$
\end{theorem}

\paragraph{Roadmap.} We present the proof of the support recovery result, i.e., Theorem~\ref{sec:theorem:support}, in Section~\ref{proof:support}. Section~\ref{proof:asymptotic} gives the proof of our main result, Theorem~\ref{sec:theorem:critical}. The argument rests on two critical lemmas (Lemmas~\ref{good_proxy} and~\ref{bounds}), whose proofs appear in the Supplementary material. In Section~\ref{sec:experiments}, we run simulations to verify Theorem~\ref{sec:theorem:critical}. We also run experiments that strongly suggest that the standard squared loss function has a local minimum in a neighborhood around $A^*$.

\newpage 
\section{A Layer of ReLU Gates can Recover the Support of the Sparse Code (Proof of Theorem \ref{sec:theorem:support})}\label{proof:support}

Most sparse coding algorithms are based on an alternating minimization approach, where one iteratively finds a sparse code based on the current estimate of the dictionary, and then uses the estimated sparse code to update the dictionary. The analogue of the sparse coding step in an autoencoder, is the passing through the hidden layer of activations of a certain affine transformation ($W$ which behaves as the current estimate of the dictionary) of the input vectors. We show that under certain stochastic assumptions, the hidden layer of ReLU gates in an autoencoder recovers with high probability the support of the sparse vector which corresponds to the present input.

\begin{proof}[Proof of Theorem~\ref{sec:theorem:support}]
From the model assumptions, we know that the dictionary $A^*$ is incoherent, and has unit norm columns. So, $ \vert \langle A_i^* , A_j^* \rangle \vert \leq \frac{\mu}{\sqrt{n}}$ for all $i \neq j$, and $||A^*_i||=1$ for all $i$. This means that for $i \neq j$,

 \begin{align}\label{eq:bound-inner-prod-WiAj} 
 \nonumber \vert \langle W_i , A_j^*\rangle \vert  &= \vert \langle W_i - A_i^*, A_j^*\rangle  \vert + \vert \langle A_i^* , A_j^* \rangle \vert \\
 &\leq || W_i - A_i^* ||_2 ||A_j^*||_2 + \frac{\mu}{\sqrt{n}} \leq  (\delta + \frac{\mu}{\sqrt{n}})
 \end{align}

~\\
Otherwise for $i = j$, 
\[\langle W_i , A_i^*\rangle   =  \langle W_i - A_i^*, A_i^*\rangle  +  \langle A_i^* , A_i^* \rangle  = \langle W_i - A_i^*, A_i^*\rangle + 1, \] 
and thus,
\begin{equation}\label{eq:bound-inner-prod-WiAi} 1 - \delta \leq \langle W_i , A_i^*\rangle \leq 1 + \delta, \end{equation}
where we use the fact that $\vert \langle W_i - A_i^*, A_i^*\rangle \vert \leq \delta$.

Let $y=A^*\x^*$ and let $S$ be the support of $\x^*$. Then we define the input to the ReLU activation $Q - \epsilon = W\y - \epsilon$ as
\begin{dmath*}
Q_i = \sum_{j \in S} \langle W_i, A^*_j \rangle x^*_j = \langle W_i, A^*_i \rangle x^*_i \mathfrak{1}_{i\in S}+ \sum_{j \in S \setminus i} \langle W_i, A^*_j \rangle x^*_j = \langle W_i, A^*_i \rangle x^*_i\mathfrak{1}_{i\in S} + Z_i.
\end{dmath*}
~\\
First we try to get bounds on $Q_i$ when $i \in \textrm{supp}(x^*)$. From our assumptions on the distribution of $x^*_i$ we have, $ 0 \leq a \leq x_i^* \leq b$ and $\mathbb{E}[x^*_i] = m_1$ for all $i$ in the support of $x^*$. For $i \in \textrm{supp}(x^*)$, 
\begin{align*}
Q_i &= \langle W_i, A^*_i \rangle x^*_i + Z_i\\
\implies Q_i &\geq (1-\delta)a + Z_i
\end{align*} 
where we use~\eqref{eq:bound-inner-prod-WiAi}. Using~\eqref{eq:bound-inner-prod-WiAj}, $Z_i$ has the following bounds:
\[ -b k \left( \delta + \frac{\mu}{\sqrt{n}} \right) \leq Z_i \leq b k \left( \delta + \frac{\mu}{\sqrt{n}} \right) \]
Plugging in the lower bound for $Z_i$ and the proposed value for the bias, we get
\begin{align*}
Q_i - \epsilon &\geq (1-\delta) a - bk \left( \delta + \frac{\mu}{\sqrt{n}} \right) - 2 m_1 k \left( \delta + \frac{\mu}{\sqrt{n}} \right)
\end{align*}
~\\
For $Q_i - \epsilon \geq 0$, we need:
\[ a \geq \frac{(b+2 m_1) \left( \delta + \frac{\mu}{\sqrt{n}} \right) k}{1- \delta}\]
Now plugging in the values for the various quantities, $\frac{\mu}{\sqrt{n}} = h^{- \xi}$ and $k = h^p$ and $\delta = O \left( h^{-p -\nu^2} \right)$, if we have $a = \Omega \left( b h^{-\nu^2} \right)$, then $Q_i - \epsilon \geq 0$. 

~\\
Now, for $i \notin \textrm{supp}(x^*)$ we would like to analyze the following probability:
\begin{equation*}
\textrm{Pr}[ Q_i -\epsilon \geq 0 \vert i \notin \textrm{supp}(x^*)]
\end{equation*}
We first simplify the quantity $\textrm{Pr}[ Q_i -\epsilon \geq 0 \vert i \notin \textrm{supp}(x^*)]$ as follows

\begin{dmath*}
\textrm{Pr}[ Q_i \geq \epsilon \vert i \notin \textrm{supp}(x^*) ] = \textrm{Pr} [ Z_i \geq \epsilon] \\
= \textrm{Pr} \left[ \sum_{j \in S\setminus i} \langle W_i, A_j^* \rangle x_j^* \geq \epsilon \right] 
\end{dmath*}
~\\
Using the Chernoff's bound, we can obtain
\begin{align*}
\textrm{Pr} [ Z_i \geq \epsilon] & \leq \underset{t \geq 0}{\textrm{inf}} e^{ -t\epsilon}\mathbb{E} \left[ \prod_{j \in S \setminus i} \left[ e^{ t \langle W_i, A_j^* \rangle x_j^*} \right] \right] \\ 
&= \underset{t \geq 0}{\textrm{inf}} e^{-t\epsilon} \prod_{j \in S \setminus i} \mathbb{E} \left[ e^{  t\langle W_i, A_j^* \rangle x_j^*} \right]  \\ 
&\leq  \underset{t \geq 0}{\textrm{inf}} e^{-t\epsilon } \mathbb{E}^k \left[ e^{  t \left(\delta + \frac{\mu}{\sqrt{n}} \right)  x^*_j} \right] \\
&\leq \underset{t \geq 0}{\textrm{inf}} e^{-t\epsilon} \left( e^{t \left( \delta + \frac{\mu}{\sqrt{n}} \right) m_1 }e^{  \frac{ t^2 \left(\delta + \frac{\mu}{\sqrt{n}} \right)^2 (b-a)^2 }{8}} \right)^k
\end{align*}
~\\
where the second inequality follows from ~\eqref{eq:bound-inner-prod-WiAj} and the fact that $t$ and $x^*_i$ are both nonnegative, and the third inequality follows from Hoeffding's Lemma. 
Next, we also have 
\begin{align*}
\textrm{Pr} [ Z_i \geq \epsilon]  &\leq \underset{t \geq 0}{\textrm{inf}} e^ { - t \left( \epsilon - k \left( \delta + \frac{\mu}{\sqrt{n}} \right) m_1 \right)  + t^2 \frac{k}{8} \left(\delta + \frac{\mu}{\sqrt{n}} \right)^2 (b-a)^2 } \\
&= e^{ - \frac{ (\epsilon -k(\delta + \frac{\mu}{\sqrt{n}}) m_1 )^2 }{ \frac{k }{2} (\delta + \frac{\mu}{\sqrt{n}})^2 (b-a)^2}}.
\end{align*}
~\\ Finally, since $k = h^p$ and $\epsilon = 2 m_1 k \left( \delta + \frac{\mu}{\sqrt{n}} \right)$, we have 
\begin{align*}
\exp \left( - \frac{2(\epsilon - k m_1 (\delta + \frac{ \mu}{\sqrt{n} } ) )^2 }{h^p (\delta + \frac{\mu}{\sqrt{n}} )^2(b -a )^2 } \right) = \exp \left( - \frac{2 h^p m_1^2}{(b-a)^2} \right) 
\end{align*}
\end{proof}

\newpage
\section{Criticality of a neighborhood of $A^*$ (Proof of Theorem \ref{sec:theorem:critical})} \label{proof:asymptotic}

It turns out that the expectation of the full gradient of the loss function~\eqref{eqn:loss} is difficult to analyze directly. Hence corresponding to the true gradient with respect to the $i^{\textrm{th}}-$column of $W^\top$ we create a proxy, denoted by $\widehat{\nabla_i L}$), by replacing in the expression for the true expectation $\nabla_i L = \mathbb{E} \left[ \frac{\partial L}{\partial W_i} \right]$ every occurrence of the random variable $\textrm{Th}(W^\top_i y - \epsilon_i) = \textrm{Th}(W^\top_i A^*x^* - \epsilon_i)$ by the indicator random variable $\mathbf{1}_{i\in \textrm{supp}(x^*)}$. This proxy is shown to be a good approximant of the expected gradient in the following lemma. 

\begin{lemma}\label{good_proxy}
Assume that the hypotheses of Theorem \ref{sec:theorem:support} hold and additionally let $b$ be bounded by a polynomial in $h$.  Then we have for each $i$ (indexing the columns of $W^\top$),
\[\Bigg \vert \Bigg \vert \widehat{\nabla_i L} - \mathbb{E} \left[ \frac{\partial L}{\partial W_i} \right] \Bigg \vert \Bigg \vert_2 \leq \textrm{poly}(h) \textrm{exp} \left(- \frac{h^p m_1^2}{2(b-a)^2} \right) \]
\end{lemma}
\begin{proof}
This lemma has been proven in Section~\ref{app:proxy} of the Appendix. 
\end{proof}
~\\



\begin{lemma}\label{bounds}
~\\
Assume that the hypotheses of Theorem \ref{sec:theorem:support} hold, and $p < \min \{ \frac{1}{2}, \nu^2\}$ (and hence $\xi > 2p$). 
Then for each $i$ indexing the columns of $W^\top$, there exist real valued functions $\alpha_i$ and $\beta_i$, and a vector $e_i$ such that $\widehat{\nabla_i L} = \alpha_i W_i - \beta_i A^*_i + e_i$, and
\begin{align*}
\alpha_i =  \Theta(m_2h^{p-1})+o(m_1^2h^{p-1})\\
\beta_i =  \Theta(m_2h^{p-1})+o(m_1^2h^{p-1})\\
\alpha_i - \beta_i = o(\max\{m_1^2,m_2\}h^{p-1})\\
\vert \vert e_i \vert \vert _2 = o(\max\{m_1^2,m_2\}h^{p-1})
\end{align*}
\end{lemma}

\begin{proof}
This lemma has been proven in Section~\ref{app:asymptotics} of the Appendix.
\end{proof}
~\\ \\
With the above asymptotic results, we are in a position to assemble the proof of Theorem~\ref{sec:theorem:critical}.

\begin{proof}[Proof of Theorem~\ref{sec:theorem:critical}] Consider any $i$ indexing the columns of $W^\top$. Recall the definition of the proxy gradient $\widehat{\nabla_i L}$ at the beginning of this section. Let us define $\gamma_i = \widehat{\nabla_i L} - \mathbb{E}\left [ \frac {\partial L}{\partial W_i}\right ]$. Using $\alpha_i, \beta_i$ and $e_i$ as defined in Lemma~\ref{bounds}, we can write the expectation of the true gradient as, $\mathbb{E}\left [ \frac {\partial L}{\partial W_i}\right ] = \alpha_i W_i - \beta_i A_i^* + e_i - \gamma_i$. Further, by Lemma~\ref{good_proxy},  
\[ \Vert \gamma_i \Vert \leq \textrm{poly}(h) \textrm{exp} \left(- \frac{h^p m_1^2}{2(b-a)^2} \right).\]  
Since $\textrm{exp} \left(\frac{h^p m_1^2}{2(b-a)^2}\right)$ is superpolynomial in $h$, we obtain 

\begin{align*}
\bigg\Vert \mathbb{E}\left [ \frac {\partial L}{\partial W_i}\right ] \bigg\Vert_2 &= \vert \vert \alpha_i W_i - \beta_i A_i^* + e_i - \gamma_i \vert \vert_2\\
&= \vert \vert \alpha_i (W_i - A_i^*) +(\alpha_i - \beta_i) A_i^* + e_i - \gamma_i \vert \vert_2\\
&\leq \vert \alpha_i \vert \Vert W_i - A_i^* \Vert_2 + \vert \alpha_i - \beta_i \vert + \vert \vert e_i - \gamma_i \vert \vert_2\\
&\leq \frac{\Theta(m_2h^{p-1})}{h^{2p+\theta^2}} + o (\max\{m_1^2,m_2\}h^{p-1})\\
&+ o(\max\{m_1^2,m_2\}h^{p-1})\\
& = o (\max\{m_1^2,m_2\}h^{p-1}) 
\end{align*}
\end{proof}

\section{Simulations}\label{sec:experiments}

We conduct some experiments on synthetic data in order to check whether the gradient norm is indeed small within the columnwise $\delta$-ball of $A^*$. We also make some observations about the landscape of the squared loss function, which has implications for being able to recover the ground-truth dictionary $A^*$.

\paragraph*{Data Generation Model:} We generate random dictionaries ($A^*$) of size $n \times h$ where $n=100$, and $h=256, 512, 1024, 2048$ and $4096$. The dictionary entries are drawn from a standard Gaussian, and the columns of the dictionary are then normalized. These dictionaries are incoherent, with high probability. For each $h$, we generate a dataset containing $N=5000$ sparse vectors with $h^p$ non-zero entries, where $p=0.01,0.02,0.05,0.1$. In our experiments, the coherence parameter $\xi$ was approximately $0.1$. We conduct experiments for values of $p$ that are at most $\xi$. Here $h$ is the hidden layer dimension of the autoencoder and $p$ controls the sparsity of the data used to train the autoencoder. The support of each sparse vector $x^*$ is drawn uniformly from all sets of indices of size $h^p$, and the non-zero entries in the sparse vectors are drawn from a uniform distribution between $a = 1$ and $b = 10$. Once we have generated the sparse vectors, we collect them in a matrix $X^* \in \mathbb{R}^{h \times N}$ and then compute the signals $Y = A^* X^*$.
~\\ \\
We set up the autoencoder as defined through equation \ref{eqn:autoencoder}. The bias parameter in the hidden layer is set to $\epsilon = 0.3 \times m_1 k \left( \delta + \frac{\mu}{\sqrt{n}} \right)$. Choosing this prefactor of $0.3$ does not violate Theorem \ref{sec:theorem:support} and it was chosen to have the ReLU layer of the autoencoder recover a large fraction of the support of $X^*$. We analyze the squared loss function in~\eqref{eqn:loss} and its gradient with respect to a column of $W$ through their empirical averages over the signals in $Y$.

\begin{table}[h]
\centering
\begin{tabular}{|l||*{4}{c|}}\hline
\backslashbox{$h$}{$p$} & 0.01 & 0.02 & 0.05 & 0.1 \\\hline\hline
256 & (0.0137, 0.0041) & (0.0138, 0.0044) & (0.0126, 0.0052) & (0.0095, 0.0068) \\\hline
512 & (0.0058, 0.0021) & (0.0058, 0.0022) & (0.0054, 0.0027) & (0.0071, 0.0036) \\\hline
1024 & (0.0025, 0.0010) & (0.0024, 0.0011) & (0.0026, 0.0014) & (0.0079, 0.0020) \\\hline
2048 & (0.0011, 0.0005) & (0.0012, 0.0006) & (0.0025, 0.0007) & (0.0031, 0.0010) \\\hline
4096 & (0.0006, 0.0003) & (0.0012, 0.0003) & (0.0013, 0.0004) & (0.0026, 0.0006) \\\hline
\end{tabular}
\caption{Average gradient norm for points that are $\frac{\delta}{2}$ away from $A^*$. For each $h$ and $p$ we report $\left( \vert \vert \mathbb{E} \left[ \frac{\partial L}{\partial W_i} \right] \vert \vert, h^{p-1}\right)$. We note that the gradient norm and $h^{p-1}$ are of the same order, and for any fixed $p$ the gradient norm is decreasing with $h$ as expected from Theorem \ref{sec:theorem:critical}}
\label{tab:grad-norm}
\end{table}

\paragraph*{Results:} Once we have generated the data, we compute the empirical average of the gradient of the loss function in~\eqref{eqn:loss} at $200$ random points which are columnwise $\frac{\delta}{2} = \frac{1}{2h^{2p}}$ away from $A^*$. We average the gradient over the $200$ points which are all at the same distance from $A^*$, and compare the average column norm of the gradient to $h^{p-1}$. Our experiments show that the average column norm of the gradient is of the order of $h^{p-1}$ (and thus falling with $h$ for any fixed $p$) as expected from Theorem \ref{sec:theorem:critical}. Results for points sampled at $\frac{\delta}{2}$ are shown in Table \ref{tab:grad-norm}.

\begin{figure}[h]
\centering
\includegraphics[scale=0.5]{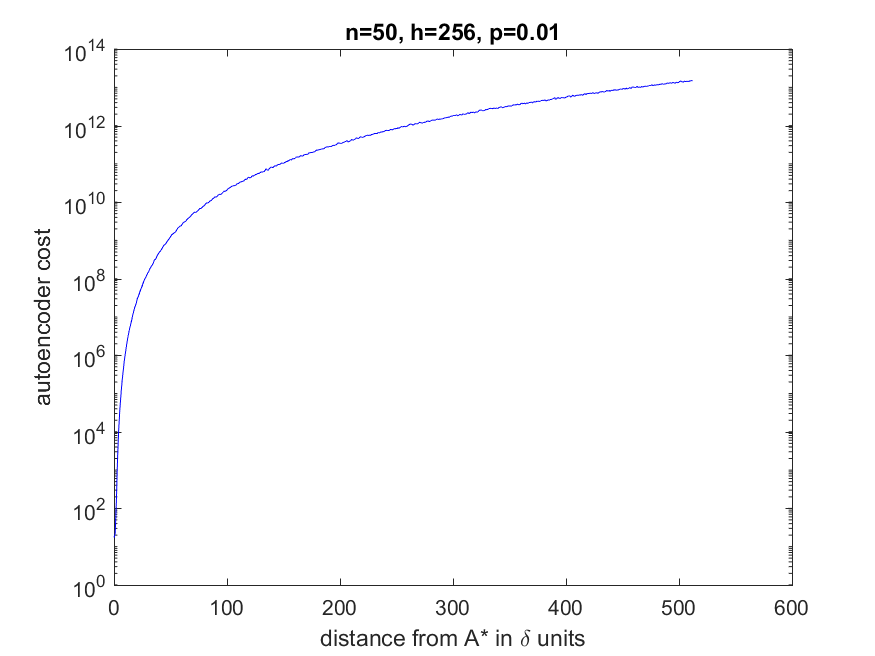}
\caption{Loss function plot for $h=256$, $p=0.01$}
\label{fig:loss_001_256}
\end{figure}

\begin{figure}[h]
\centering
\includegraphics[scale=0.5]{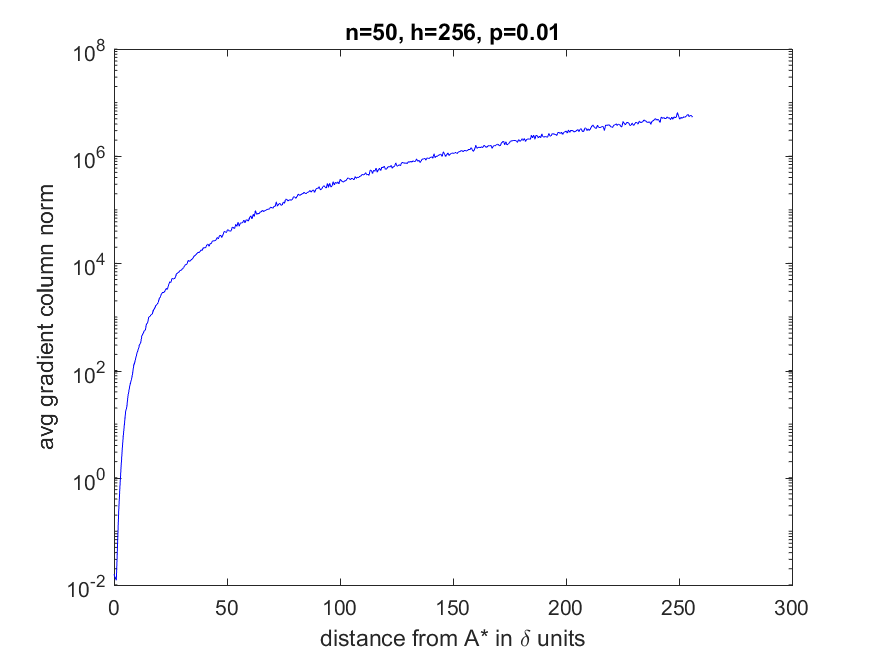}
\caption{Average gradient norm plot for $h=256$, $p=0.01$}
\label{fig:grad_001_256}
\end{figure}

\begin{figure}[h]
\centering
\includegraphics[scale=0.5]{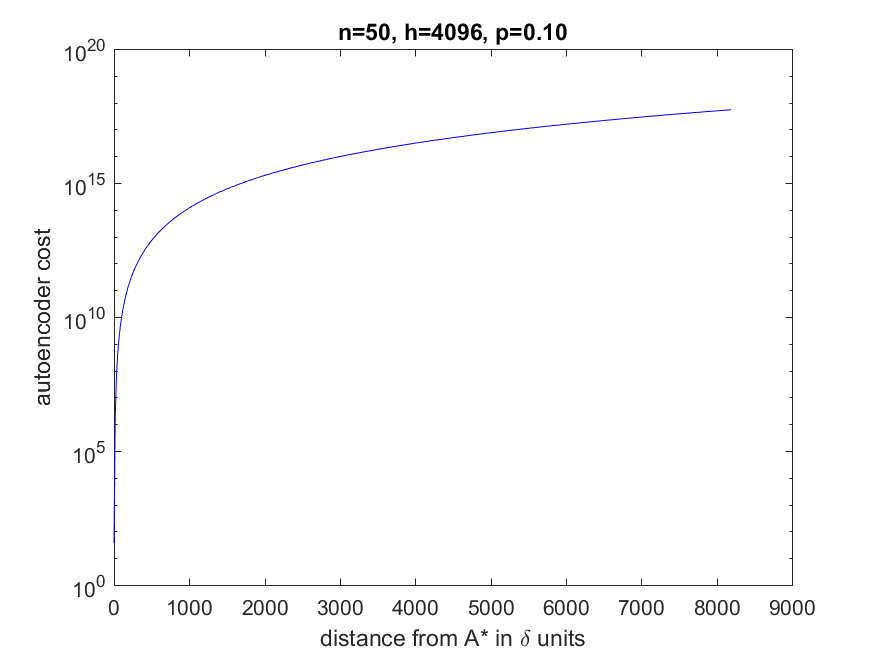}
\caption{Loss function plot for $h=4096$, $p=0.1$}
\label{fig:loss_01_4096}
\end{figure}

\begin{figure}[h]
\centering
\includegraphics[scale=0.5]{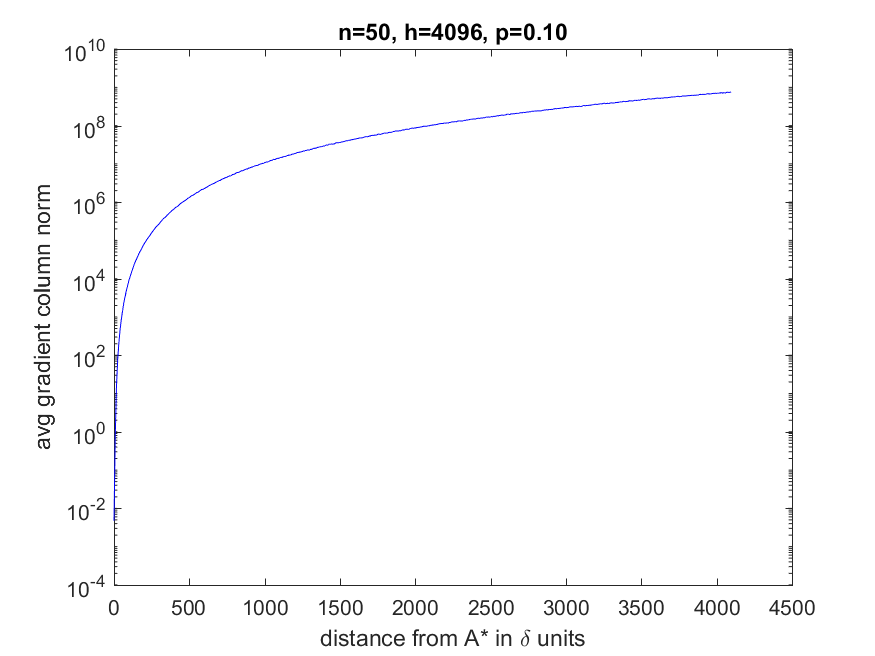}
\caption{Average gradient norm plot for $h=4096$, $p=0.1$}
\label{fig:grad_01_4096}
\end{figure}

\begin{figure}[h]
\centering
\includegraphics[scale=0.5]{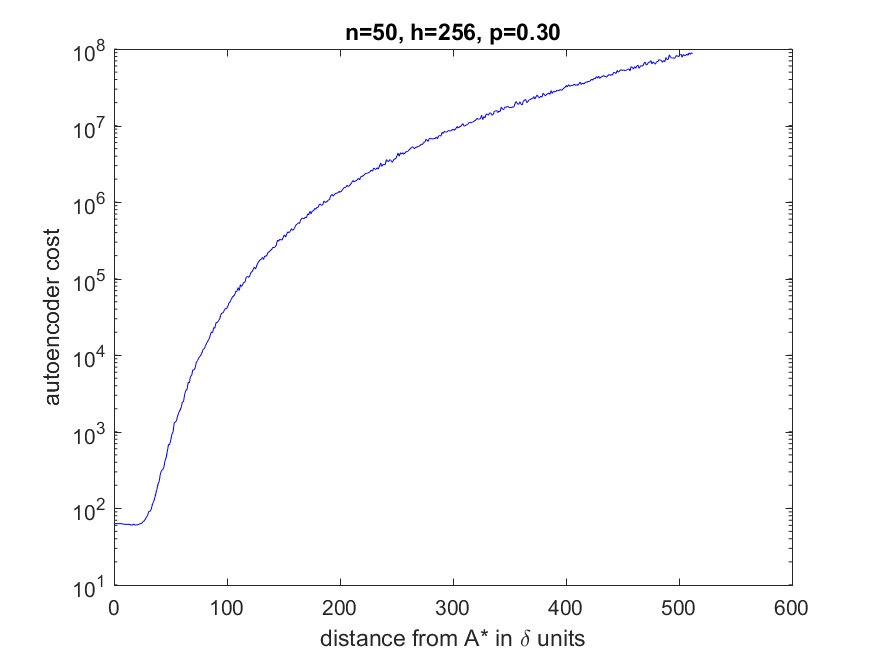}
\caption{Loss function plot for $h=256$, $p=0.3$}
\label{fig:loss_03_256}
\end{figure}

\begin{figure}[h]
\centering
\includegraphics[scale=0.5]{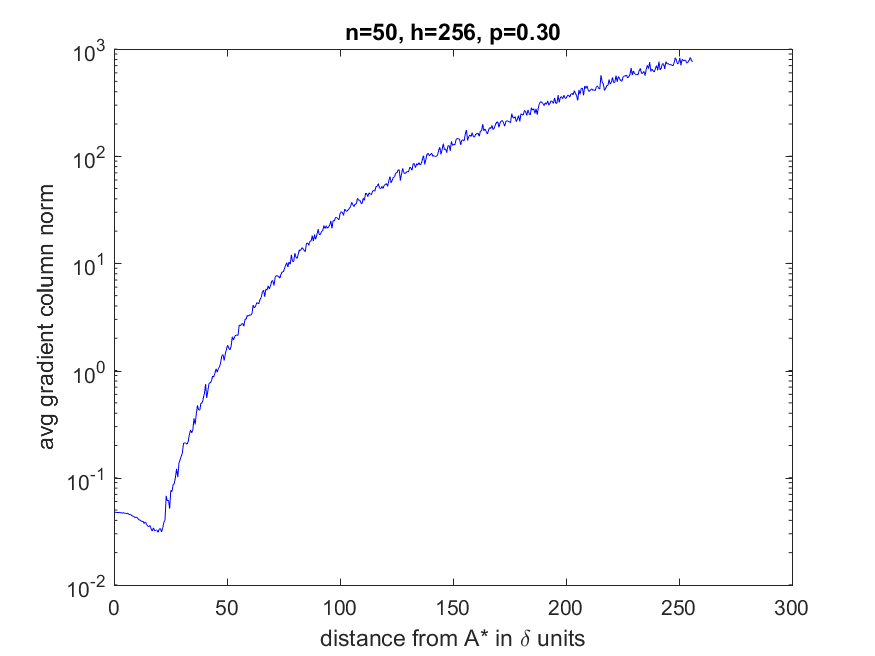}
\caption{Average gradient norm plot for $h=256$, $p=0.3$}
\label{fig:grad_03_256}
\end{figure}
~\\ \\
We also plot the squared loss of the autoencoder along a randomly chosen direction to see if $A^*$ is possibly a local minimum. More precisely, we draw a matrix $\Delta W$ from a standard normal distribution, and normalize its columns. We then plot $f(t) = L( (A^* + t \Delta W )^\top )$, as well as the gradient norm averaged over all the columns. For purposes of illustration, we show these plots for $h=256$, $p=0.01$ in figures \ref{fig:loss_001_256} and \ref{fig:grad_001_256}, and those for $h=4096$, $p=0.1$ in figures \ref{fig:loss_01_4096} and \ref{fig:grad_01_4096}. 
~\\ \\
From the first four plots, we can observe that the loss function value, and the gradient norm keeps decreasing as we get close to $A^*$. Since $\Delta W$ is a randomly chosen direction, this suggests that $A^*$ is a local minimum for the squared loss function. The plots we show here are in the log-scale along the y-axis, which is why it seems as though there is a sharp decrease in the function value. Viewed in normal scale, the function seems to decrease smoothly to a local minimum at $A^*$.
~\\ \\
In figures \ref{fig:loss_03_256} and \ref{fig:grad_03_256} we plot the function and gradient norm for $h=256$ and $p=0.3$. This value of $p$ is much larger than the coherence parameter $\xi$, and hence outside the region where the support recovery result, Theorem \ref{sec:theorem:support} is valid. We suspect that $A^*$ is now in a region where $\textrm{ReLU}(A^{*\top}y - \epsilon ) = 0$, which means the function is flat in a small neighborhood of $A^*$.

\section{Conclusion}
In this paper we have undertaken a rigorous analysis of the loss function of the squared loss of an autoencoder when the data is assumed to be generated by sensing of sparse high dimensional vectors by an overcomplete dictionary. {\bf We have shown that the expected gradient of this loss function is very close to zero in a neighborhood of the generating overcomplete dictionary.} 
~\\ \\
Our simulations complement this theoretical result by providing further empirical support. Firstly, they show that the gradient norm in this $\delta-$ball of $A^*$ indeed falls with $h$ and is of the same order as $\frac{1}{h^{1-p}}$ as expected from our proof. Secondly, the experiments also strongly suggest ranges of values of $h$ and $p$ where $A^*$ is a local minima of this loss function and that it has a neighborhood where the reconstruction error is low.
~\\ \\
This suggests sparse coding problems can be solved by training autoencoders using gradient descent based algorithms. 
Further, recent investigations have led to the conjecture/belief that many important unsupervised learning tasks, e.g. recognizing handwritten digits, are sparse coding problems in disguise~\cite{makhzani2013k,makhzani2015winner}. Thus, our results could shed some light on the observed phenomenon that gradient descent based algorithms train autoencoders to low reconstruction error for natural data sets, like MNIST.
~\\ \\
It remains to rigorously show whether a gradient descent algorithm can be initialized randomly (may be far away from $A^*$) and still be shown to converge to this neighborhood of critical points around the dictionary. Towards that it might be helpful to understand the structure of the Hessian outside this neighborhood. Since our analysis applies to the expected gradient, it remains to analyze the sample complexities where these nice results will become prominent. 
~\\ \\
The possibility also remains open that this standard loss or some other loss functions exist for the autoencoder with the provable property of having a global minima/minimum at the ground truth dictionary. We have mentioned one example of such in a special case (when $A^*$ is square orthogonal and $x^*$ is nonnegative) and even in this special case it remains open to find a provable optimization algorithm. 
~\\ \\
On the simulation front we have a couple of open challenges yet to be tackled. Firstly, it is left to find efficient implementations of the iterative update rule based on the exact gradient of the proposed loss function which has been given in~\eqref{eqn:loss}. This would open up avenues for testing the power of this loss function on real data rather than the synthetic data used here. Secondly, a simulation of the main Theorem~\ref{sec:theorem:critical} that can probe deeper into its claim would need to be able to sample $A^*$ for different $h$ at a fixed value of the incoherence parameter $\xi$. This sampling question of $A^*$ with these constraints is an unresolved one that is left for future work. 
~\\ \\Autoencoders with more than one hidden layer have been used for unsupervised feature learning \cite{le2013building} and recently there has been an analysis of the sparse coding performance of convolutional neural networks with one layer~\cite{gilbert2017towards} and two layers of nonlinearities \cite{vardan2016convolutional}. The connections between neural networks and sparse coding has also been recently explored in~\cite{bora2017compressed}. It remains an exciting open avenue of research to try to do a similar study as in this work to determine if and how deeper architectures under the same generative model might provide better means of doing sparse coding.

\section*{Acknowledgements}
Akshay Rangamani and Peter Chin are supported by the AFOSR grant FA9550-12-1-0136. Amitabh Basu and Anirbit Mukherjee gratefully acknowledges support from the NSF grant CMMI1452820. We would like to thank Raman Arora (JHU), and Siva Theja Maguluri (Georgia Institute of Technology) for illuminating comments and discussion.

\bibliographystyle{abbrv}
\bibliography{references}

\newpage
\appendix 
\section*{Appendix}
\section {The proxy gradient is a good approximation of the true expectation of the gradient (Proof of Lemma 5.1)} \label{app:proxy}

\begin{proof}
To make it easy to present this argument let us abstractly think of the function $f$ (defined for any $i \in \{1,2,3,..,h\}$) as $f(y,W,X) = \frac{\partial L}{\partial W_i} $ where we have defined the random variable $X = \text{Th}[W_i^Ty-\epsilon_i]$. It is to be noted that because of the  ReLU term and its derivative this function $f$ has a dependency on $y= A^*x^*$ even outside its dependency through $X$. Let us define another random variable $Y = \mathbf{1}_{i \in \text{Support}(x^*)}$. Then we have,
\begin{align*}
&\big\Vert \mathbb{E}_{x^*}[f(y,W,X)] - \mathbb{E}_{x^*}[f(y,W,Y)] \big\Vert_{\ell_2} \\
\leq  &\mathbb{E}_{x^*} [ \vert f(y,W,X) -  f(y,W,Y) \vert_{\ell_2} ]\\
\leq &\mathbb{E}_{x^*} [ \vert f(y,W,X) (\mathbf{1}_{X=Y} + \mathbf{1}_{X \neq Y}) -  f(y,W,Y)(\mathbf{1}_{X=Y} + \mathbf{1}_{X \neq Y}) \vert_{\ell_2} ] \\
\leq & \mathbb{E}_{x^*} [\vert (f(y,W,X) - f(y,W,Y)) \vert_{\ell_2}\mathbf{1}_{X \neq Y} ]\\
\leq &\sqrt{\mathbb{E}_{x^*}[ \big\vert f(y,W,X) - f(y,W,Y) \big\vert_{2}^2 ]} \sqrt{\mathbb{E}_{x^*} [\mathbf{1}_{X \neq Y}]}
\end{align*}
~\\
In the last step above we have used the Cauchy-Schwarz inequality for random variables. We recognize that $\mathbb{E}_{x^*}[f(y,W,Y)]$ is precisely what we defined as the proxy gradient $\widehat{\nabla_i L}$. Further for such $W$ as in this lemma the support recovery theorem (Theorem 3.1) holds and that is precisely the statement that the term, $\mathbb{E}_{x^*} [\mathbf{1}_{X \neq Y}]$ is small. So we can rewrite the above inequality as, 
\begin{align*}
\bigg\Vert \mathbb{E}_{x^*}[\frac{\partial L}{\partial W_i}] - \widehat{\nabla_i L} \bigg\Vert_{2} \leq \sqrt{\mathbb{E}_{x^*}[ \big\vert f(y,W,X) - f(y,W,Y) \big\vert_{2}^2 ]} \exp \left ( - \frac{h^pm_1^2}{2(b-a)^2} \right )
\end{align*}
~\\
We remember that $f$ is a polynomial in $h$ because its $h$ dependency is through Frobenius norms of submatrices of $W$ and $\ell_2$ norms of projections of $Wy$. But the $\ell_\infty$ norm of the training vectors $y$ (that is $b$) have been assumed to be bounded by $\text{poly}(h)$. Also we have the assumption that the columns of $W^\top$ are within a $\frac{1}{h^{p+\nu^2}}-$ball of the corresponding columns of $A^*$ which in turn is a $n \times h$ dimensional matrix of bounded norm because all its columns are normalized. So summarizing we have, 

\begin{align*}
\bigg\Vert \mathbb{E}_{x^*}[\frac{\partial L}{\partial W_i}] - \widehat{\nabla_i L} \bigg\Vert_{2} \leq  \text{poly}(h)\exp \left ( - \frac{h^pm_1^2}{2(b-a)^2} \right )
\end{align*}
~\\
The above inequality immediately implies the claimed lemma.
\end{proof}

\newpage
\section {The asymptotics of the coefficients of the gradient of the squared loss (Proof of Lemma $5.2$)}\label{app:asymptotics}

To recap we imagine being given as input signals $y \in \mathbb{R}^n$ (imagined as column vectors), which are generated from an overcomplete dictionary $A^* \in \mathbb{R}^{n \times h}$ of a fixed incoherence. Let $x^* \in \mathbb{R}^h$ (imagined as column vectors) be the sparse code that generates $y$.
The model of the autoencoder that we now have is $\hat{y} = W^\top \textrm{ReLU}(Wy - \epsilon)$. $W$ is a $h \times n$ matrix and the $i^{th}$ column of $W^\top$ is to be denoted as the column vector $W_i$. 

\subsection {Derivative of the standard squared loss of a ReLU autoencoder}
Using the above notation the squared loss of the autoencoder is $\frac{1}{2} \vert \vert \hat{y} -  y \vert \vert^2$. But we introduce a dummy constant $D=1$ to be multiplied to $y$ because this helps read the complicated equations that would now follow. This marker helps easily spot those terms which depend on the sensing of $x^*$ (those with a factor of $D$) as opposed to the terms which are ``purely'' dependent on the neural net (those without the factor of $D$). Thus we think of the squared loss $L$ of our autoencoder as, 
\[ L = \frac{1}{2} \vert \vert \hat{y} - D y \vert \vert^2 = \frac {1}{2} (W^\top \textrm{ReLU}(Wy - \epsilon) - D y)^\top (W^\top \textrm{ReLU}(Wy - \epsilon) - D y) = \frac {1}{2} f^T f\] 
~\\
where we have defined $f \in \R^n$ as, 
\[ f = W^\top \textrm{ReLU}(Wy - \epsilon) - D y\]
Then we have, 
\[ J_{W_i}(f)_{ab} = \frac {\partial f_a}{\partial W_{ib}} = \textrm{ReLU}(W_i^\top y - \epsilon)\delta_{ab} + \textrm{Th} (W_i^Ty - \epsilon) W_{ia}y_b\]
In the form of a $n \times n$ derivative matrix this means,
\[ J_{W_i}(f) = \left[ \frac {\partial f_a}{\partial W_{ib}} \right] = \textrm{ReLU}(W_i^\top y - \epsilon)I + \textrm{Th} (W_i^\top y - \epsilon)W_iy^\top\]

~\\
This helps us write, 

\begin{align*}
\frac {\partial L}{\partial W_i} &= J_{W_i}(f))^\top f\\
&= (\textrm{ReLU}(W_i^\top y - \epsilon)I + \textrm{Th} (W_i^\top y - \epsilon)W_iy^\top)^\top [W^\top \textrm{ReLU}(Wy - \epsilon) - D y ]\\
&= \textrm{Th}(W^\top_i y - \epsilon_i) \left[ (W_i^\top y - \epsilon_i)I + y W_i^\top \right] \left( \sum_{j=1}^{h} \textrm{ReLU}(W_j^\top y - \epsilon_j) W_j - D y \right) \\
\end{align*}

~\\
Now going over to the proxy gradient $\widehat{\nabla_i L}$ corresponding to this term we define the vector $G_i$ as, 

\begin{align*}
\widehat{\nabla_i L} &= \mathbb{E}_{S \in \mathbb{S}} \left[ \mathbf{1}_{i \in S} \times \mathbb{E}_{x^*_S} \left[ \left[ (W_i^{\top} y - \epsilon_i)I + y  W_i^\top \right] \left( \sum_{j \in S} (W_j^{\top} y - \epsilon_j) W_j - D y \right) \right] \right]\\
&= \mathbb{E}_{S \in \mathbb{S}} \left[ \mathbf{1}_{i \in S} \times G_i \right]
\end{align*}

~\\

\newpage 
Thus we have, 

\begin{align*}
G_i &= \mathbb{E}_{x^*_S} \left[ \left[ (W_i^{\top} A^* x^* - \epsilon_i)I + (A^* x^*) W_i^\top \right] \left( \sum_{j \in S} (W_j^{\top} A^* x^* - \epsilon_j) W_j - D A^* x^* \right) \right] \\
&= \underbrace{\mathbb{E}_{x^*_S} \left[ (W_i^{\top} A^* x^* - \epsilon_i)\left( \sum_{j \in S} (W_j^{\top} A^* x^* - \epsilon_j) W_j - D A^* x^* \right) \right]}_{\textrm{Term 1}} \\ 
&+ \underbrace{\mathbb{E}_{x^*_S} \left[ (A^* x^*) W_i^\top \left( \sum_{j \in S} (W_j^{\top} A^* x^* - \epsilon_j) W_j - D A^* x^* \right) \right]}_{\textrm{Term 2}}\\
&= \underbrace{\mathbb{E}_{x^*_S} \left[ \sum_{j \in S} \epsilon_i \epsilon_j W_j - \sum_{j , k \in S} \epsilon_i (W_j^{\top} A^*_k) W_j  x_k^* - \sum_{j ,k \in S} \epsilon_j (W_i^{\top} A^*_k) W_j  x_k^* + \sum_{j,k,l \in S} ( W_i^{\top} A^*_k)( W_j^{\top} A^*_l) W_j x_l^* x_k^* \right]}_{\textrm{From Term 1}} \\
&+ \underbrace{\mathbb{E}_{x^*_S} \left[ - D \sum_{j,k \in S} ( W_i^{\top} A_k^*) A_j^* x_k^* x_j^* + D \sum_{j \in S} \epsilon_i A^*_j x^*_j \right]}_{\textrm{From Term 1}} + \underbrace{\mathbb{E}_{x^*_S} \left[ - D \sum_{j,k \in S} (A_k^{* \top}W_i) A^*_j x^*_k x^*_j \right]}_{\textrm{From Term 2}} \\ 
&+ \underbrace{\mathbb{E}_{x^*_S} \left[ - \sum_{j,k \in S} \epsilon_j A_k^* (W_i^\top W_j)  x^*_k \right]}_{\textrm{From Term 2}} + \underbrace{\mathbb{E}_{x^*_S} \left[ \sum_{j,k,l \in S}  (W_i^\top W_j) (W_j^{\top} A_l^*)A_k^* x_k^* x_l^* \right]}_{\textrm{From Term 2}}
\end{align*}

\newpage 
~\\
Now we invoke the distributional assumption about i.i.d sampling of the coordinates for a fixed support and the definition of $m_1$ and $m_2$ to write, $\mathbb{E}_{x^*_S}[x^*_ix^*_j] = \mathbb{E}^2_{x^*_S}[x^*_i] = m_1^2$ for all $i \neq j$ and for $i=j$, $m_2 = \mathbb{E}_{x^*_S}[x^*_ix^*_j]$. Thus we get, 

\begin{align*}
G_i &= \underbrace{\sum_{j \in S} \epsilon_i \epsilon_j W_j - m_1 \sum_{j,k \in S}  ( W_j^\top A^*_k)  W_j \epsilon_i  - m_1 \sum_{j,k \in S} \epsilon_j (W_i^\top A^*_k) W_j}_{G^1_i\textrm{ From Term 1}}   \\
&+ \underbrace{ m_2 \sum_{j , k \in S} ( W_i ^\top A^*_k) (W_j^\top A^*_k)  W_j  + m_1^2\sum_{\substack{j, k, l \in S \\ k \neq l}} ( W_i ^\top A^*_k ) ( W_j ^\top A^*_l)  W_j}_{G^2_i \textrm{ From Term 1}} \\
&+ \underbrace{\left[ - D m_1^2\sum_{\substack{j, k \in S \\ j \neq k}} ( W_i^{\top} A_k^*) A_j^*  - D m_2 \sum_{j\in S} ( W_i^{\top} A_j^*) A_j^*  + m_1 D \sum_{j \in S} \epsilon_i A^*_j  \right]}_{G^3_i \textrm{ From Term 1}} \\ 
&-\underbrace{\left[ D m_1^2\sum_{\substack{j, k \in S \\ j \neq k}} (A_k^{* \top}W_i) A^*_j  +D m_2\sum_{j \in S} (A_j^{* \top}W_i) A^*_j \right]}_{G^4_i \textrm{ From Term 2}} \\
&-\underbrace{m_1\left [ \sum_{j,k \in S} \epsilon_j (W_i^\top W_j) A_k^*   \right] +  \left[ m_2 \sum_{j,k \in S}  (W_i^\top W_j) (W_j^{\top} A_k^*)A^*_k + m_1^2 \sum_{\substack{j, k, l \in S \\ k \neq l}}  (W_i^\top W_j) (W_j^{\top} A_l^*)A^*_k \right]}_{G^5_i \textrm{ From Term 2}}
\end{align*}
~\\
Each term in the above sum is a vector. Now we separate out from the sums the terms which are in the directions of $W_i$ or $A_i^*$ and the rest. We remember that this is being under the condition that $i \in S$. To make this easy to read we do this separation for each line of the above equation separately in a different equation block. Also inside every block we do the separation for each summation term in a separate line.

\begin{align*}
G^1_i &= \sum_{j \in S} \epsilon_i \epsilon_j W_j - m_1 \sum_{j,k \in S}  ( W_j^\top A^*_k)  W_j \epsilon_i  - m_1 \sum_{j,k \in S} \epsilon_j (W_i^\top A^*_k) W_j \\ 
&= \left [ \epsilon_i^2 W_i + \sum_{\substack{j \in S \\ j \neq i}} \epsilon_i \epsilon_j W_j \right ]\\ 
&-m_1 \left [ \sum_{k \in S}  \epsilon_i  ( W_i^\top A^*_k)  W_i + \sum_{\substack{j, k \in S \\ j \neq i}}  ( W_j^\top A^*_k)  W_j \epsilon_i\right ]\\ 
&-m_1 \left [ \sum_{k \in S} \epsilon_i (W_i^\top A^*_k) W_i + \sum_{\substack{j, k \in S \\ j \neq i}} \epsilon_j (W_i^\top A^*_k) W_j \right ]\\ 
~\\ \\
G^2_i &= m_2 \sum_{j , k \in S} ( W_i ^\top A^*_k) (W_j^\top A^*_k)  W_j  + m_1^2\sum_{\substack{j, k, l \in S \\ k \neq l}}  ( W_i ^\top A^*_k ) ( W_j ^\top A^*_l)  W_j\\
&= m_2 \left [ \sum_{ k \in S} ( W_i ^\top A^*_k) (W_i^\top A^*_k)  W_i + \sum_{\substack{j, k \in S \\ j \neq i}} ( W_i ^\top A^*_k) (W_j^\top A^*_k)  W_j \right ] \\ 
&+ m_1^2 \left [ \sum_{\substack{k, l \in S \\ k \neq l}}  ( W_i ^\top A^*_k ) ( W_i ^\top A^*_l)  W_i + \sum_{\substack{j, k, l \in S \\ j \neq i \\ k \neq l}}  ( W_i ^\top A^*_k ) ( W_j ^\top A^*_l)  W_j \right ]\\
~\\ \\
G^3_i &= -D\left[ m_1^2\sum_{\substack{j, k \in S \\ j \neq k}} ( W_i^{\top} A_k^*) A_j^*  + m_2 \sum_{j\in S} ( W_i^{\top} A_j^*) A_j^*  - m_1  \sum_{j \in S} \epsilon_i A^*_j  \right] \\
&=-D \left [ m_1^2 \sum_{\substack{k \in S \\ k \neq i}} ( W_i^{\top} A_k^*) A_i^* +  m_1^2 \sum_{\substack{j, k \in S \\ j \neq i \\ j \neq k}} ( W_i^{\top} A_k^*) A_j^* \right ]\\
&-D \left [ m_2 ( W_i^{\top} A_i^*) A_i^* + m_2 \sum_{\substack{j \in S \\ j \neq i}} ( W_i^{\top} A_j^*) A_j^* \right ] \\
&-D \left [ -m_1 \epsilon_i A_i^* - m_1 \sum_{\substack{j \in S \\ j \neq i}} \epsilon_i A_j^* \right ]\\
\end{align*}
\begin{align*}
G^4_i &= -\left[ D m_1^2\sum_{\substack{j,k \in S \\ j \neq k}} (A_k^{* \top}W_i) A^*_j  +D m_2\sum_{j \in S} (A_j^{* \top}W_i) A^*_j \right] \\
&=-D \left [ m_1^2  \sum_{\substack{k \in S \\ k \neq i}} (A_k^{* \top}W_i) A^*_i  +  m_1^2 \sum_{\substack{j ,k\in S \\ j \neq k \\ j \neq i}} (A_k^{* \top}W_i) A^*_j \right ]\\
&-D \left [ m_2 (A_i^{* \top}W_i) A^*_i + m_2 \sum_{\substack{j \in S \\ j \neq i}} (A_j^{* \top}W_i) A^*_j \right ]\\
~\\ \\
G^5_i &= -m_1\left [ \sum_{j,k \in S} \epsilon_j (W_i^\top W_j) A_k^*   \right] +  \left[ m_2 \sum_{j,k \in S}  (W_i^\top W_j) (W_j^{\top} A_k^*)A^*_k + m_1^2 \sum_{\substack{j, k, l \in S \\ k \neq l}}  (W_i^\top W_j) (W_j^{\top} A_l^*)A^*_k \right] \\
&= -m_1 \sum_{j \in S} \epsilon_j (W_i^\top W_j) A^*_i - m_1 \sum_{\substack{j,k \in S \\ k \neq i}} \epsilon_j (W_i^\top W_j) A^*_k \\
&+ m_2 \sum_{j \in S} (W_i^\top W_j) (W_j^\top A^*_i) A^*_i + m_2 \sum_{\substack{j,k \in S \\ k \neq i}} (W_i^\top W_j) (W_j^\top A^*_k) A^*_k \\
&+ m_1^2 \sum_{\substack{j,l \in S \\ l \neq i}} (W_i^\top W_j) (W_j^{\top} A_l^*)A^*_i + m_1^2 \sum_{\substack{j,k,l \in S \\ k \neq i,l}} (W_i^\top W_j) (W_j^{\top} A_l^*)A^*_k 
\end{align*}

\newpage 
~\\
So combining the above we have, 

\[ \widehat{\nabla_i L} = \alpha_i W_i - \beta_i A^*_i + e_i \]

~\\
where,
\begin{align*}
\alpha_i &=  \mathbb{E}_{S \in \mathbb{S}} \Bigg[ \mathbf{1}_{i \in S} \times \Bigg \{ m_2 \sum_{ k \in S} ( W_i ^\top A^*_k) (W_i^\top A^*_k) + m_1^2 \sum_{\substack{k, l \in S \\ k \neq l}}  ( W_i ^\top A^*_k ) ( W_i ^\top A^*_l) -2m_1 \sum_{k \in S}  \epsilon_i  ( W_i^\top A^*_k) + \epsilon_i^2 \Bigg \} \Bigg]\\ 
\beta_i &=  \mathbb{E}_{S \in \mathbb{S}} \Bigg[ \mathbf{1}_{i \in S} \times \Bigg \{  2 D m_1^2 \sum_{\substack{k \in S \\ k \neq i}} ( W_i^{\top} A_k^*) + 2 D m_2 ( W_i^\top A^*_i) - D m_1 \epsilon_i + m_1 \sum_{j \in S} \epsilon_j (W_i^\top W_j) - m_2 \sum_{j \in S} (W_i^\top W_j) (W_j^\top A^*_i) \\
&- m_1^2 \sum_{\substack{j,l \in S \\ l \neq i}} (W_i^\top W_j) (W_j^{\top} A_l^*) \Bigg \} \Bigg ] \\ 
e_i &= \mathbb{E}_{S \in \mathbb{S}} \Bigg[ \mathbf{1}_{i \in S} \times \Bigg \{ \sum_{\substack{j \in S \\ j \neq i}} \epsilon_i \epsilon_j W_j - m_1 \sum_{\substack{j, k \in S \\ j \neq i}}  \epsilon_i ( W_j^\top A^*_k)  W_j - m_1 \sum_{\substack{j, k \in S \\ j \neq i}} \epsilon_j (W_i^\top A^*_k) W_j \\
&+ m_2 \sum_{\substack{j, k \in S \\ j \neq i}} ( W_i ^\top A^*_k) (W_j^\top A^*_k)  W_j + m_1^2 \sum_{\substack{j, k, l \in S \\ j \neq i \\ k \neq l}}  ( W_i ^\top A^*_k ) ( W_j ^\top A^*_l)  W_j \\
&-2D m_1^2 \sum_{\substack{j, k \in S \\ j \neq i \\ j \neq k}} ( W_i^{\top} A_k^*) A_j^* -2D m_2 \sum_{\substack{j \in S \\ j \neq i}} ( W_i^{\top} A_j^*) A_j^* + D m_1 \sum_{\substack{j \in S \\ j \neq i}} \epsilon_i A_j^* \\
&- m_1 \sum_{\substack{j,k \in S \\ k \neq i}} \epsilon_j (W_i^\top W_j) A^*_k + m_2 \sum_{\substack{j,k \in S \\ k \neq i}} (W_i^\top W_j) (W_j^\top A^*_k) A^*_k + m_1^2 \sum_{\substack{j,k,l \in S \\ k \neq i,l}} (W_i^\top W_j) (W_j^{\top} A_l^*)A^*_k \Bigg \} \Bigg ]
\end{align*}
 
~\\ \\
We will now estimate bounds on each of the terms $\alpha_i, \beta_i, ||e_i||$. We will separate them as $\alpha_i = \tilde{\alpha_i} + \hat{\alpha_i}$ (similarly for the other terms). Where the tilde terms are those that come as a coefficient of $m_2$, and the hat terms are the ones that come as coefficient of $m_1$ or $\epsilon$ or both.

\newpage

\subsection {Estimating the $m_2$ dependent parts of the derivative}

Since $||A^*_i||=1$ and $W_i$ is being assumed to be within a $0 < \delta <1$ ball of $A^*_i$ we can use the following inequalities:
\begin{align*}
||W_i|| &= ||W_i - A^*_i + A^*_i|| \leq ||W_i - A^*_i|| + ||A^*_i|| = \delta + 1\\ 
||W_i|| &\geq 1-\delta \\
\langle W_i, A^*_i \rangle &= \langle W_i - A^*_i, A^*_i \rangle + \langle A^*_i, A^*_i \rangle \leq ||W_i - A^*_i||||A^*_i|| +  1 \leq \delta + 1\\ 
\langle W_i, A^*_i \rangle &\geq 1-\delta \\
|\langle W_j, A^*_i \rangle| &= |\langle W_j - A^*_j, A^*_i \rangle + \langle A^*_j, A^*_i \rangle| \leq \frac{\mu}{\sqrt{n}}+ ||W_j - A^*_j||||A^*_i|| = \frac{\mu}{\sqrt{n}}+\delta\\
\vert \langle W_i, W_j  \rangle \vert &= \vert \langle W_i - A_i^*,W_j \rangle + \langle A_i^*,W_j \rangle \vert \leq \delta(1+\delta) + (\delta + \frac {\mu}{\sqrt{n}}) = \delta ^2 + 2\delta + \frac{\mu}{\sqrt{n}}\\
\langle W_i, W_i \rangle &= ||W_i||^2 \geq (1-\delta)^2 \\
\langle W_i, W_i \rangle &= ||W_i||^2 \leq (1+\delta)^2
\end{align*}

\paragraph{Bounding $\tilde{\beta_i}$}
\begin{align*}
\tilde{\beta_i} &= \mathbb{E}_{S \in \mathbb{S}}\left [ \mathbf{1}_{i \in S} \left \{ 2D   m_2 ( W_i^{\top} A_i^*) - m_2 \sum_{j \in S} (W_i^\top W_j) (W_j^\top A^*_i)\right \} \right ]\\
&= \mathbb{E}_{S \in \mathbb{S}}\left [ \mathbf{1}_{i \in S} \left \{ 2D   m_2 \langle W_i, A_i^* \rangle  - m_2 ||W_i||^2 \langle W_i, A^*_i\rangle - m_2 \sum_{\substack{j \in S \\ j \neq i}} \langle W_i, W_j \rangle  \langle W_j, A^*_i\rangle \right \} \right ]
\end{align*}
~\\
Evaluating the outer expectation we get,
\begin{align}\label{beta_m2}
\nonumber \tilde{\beta_i} &= \sum_{\{S \in \mathbb{S} : i \in S\}} q_S 2D m_2 \langle W_i, A_i^* \rangle  - \sum_{\{S \in \mathbb{S} : i \in S\}} q_S m_2 ||W_i||^2 \langle W_i, A^*_i\rangle - m_2 \sum_{\substack{j =1 \\ j \neq i}}^h \langle W_i, W_j \rangle  \langle W_j , A^*_i\rangle \sum_{\{S \in \mathbb{S} : i,j \in S, i \neq j\}} q_S\\ 
\nonumber &= 2D q_i m_2 \langle W_i, A_i^* \rangle - q_i m_2 ||W_i||^2 \langle W_i, A^*_i\rangle - m_2 \sum_{\substack{j =1 \\ j \neq i}}^h q_{ij} \langle W_i, W_j \rangle  \langle W_j , A^*_i\rangle\\
\newline 
\nonumber &\text{Upper bounding the above we get,}\\
\nonumber \tilde{\beta_i} &\leq 2D m_2 h^{p-1} (1+\delta) - m_2 h^{p-1} (1-\delta)^3 + m_2 h^{2p-1} \left( \delta + \frac{\mu}{\sqrt{n}} \right) \left( \delta^2 + 2\delta + \frac{\mu}{\sqrt{n}} \right)\\
\nonumber &= 2D m_2 h^{p-1} (1+h^{-p-\nu^2}) - m_2 h^{p-1} (1-3h^{-p-\nu^2} + 3 h^{-2p-2\nu^2} - h^{-3p-3\nu^2}) \\
&+ m_2 h^{2p-1} (h^{-3p-3\nu^2} + 2h^{-2p-2\nu^2} + h^{-2p-2\nu^2 - \xi} + 3h^{-p-\nu^2 - \xi} + h^{-2 \xi}) \\
\nonumber &\text{Similarly for the lower bound on $\beta_i$ we get,}\\
\nonumber \tilde{\beta_i} &\geq 2D m_2 h^{p-1} (1-\delta) - m_2 h^{p-1} (1+\delta)^3 - m_2 h^{2p-1} \left( \delta + \frac{\mu}{\sqrt{n}} \right) \left( \delta^2 + 2\delta + \frac{\mu}{\sqrt{n}} \right)\\
\nonumber &= 2D m_2 h^{p-1} (1-h^{-p-\nu^2}) - m_2 h^{p-1} (1+3h^{-p-\nu^2} + 3 h^{-2p-2\nu^2} + h^{-3p-3\nu^2}) \\
&- m_2 h^{2p-1} (h^{-3p-3\nu^2} + 2h^{-2p-2\nu^2} + h^{-2p-2\nu^2 - \xi} + 3h^{-p-\nu^2 - \xi} + h^{-2 \xi})
\end{align}
~\\
Thus for $0<p<2\xi$ and $D=1$, we have $\beta = \Theta \left( m_2 h^{p-1} \right)$

\paragraph{Bounding $\tilde{\alpha_i}$}
\begin{align*}
\tilde{\alpha_i} &= \mathbb{E}_{S \in \mathbb{S}}\left [ \mathbf{1}_{i \in S} \left \{ m_2  \sum_{k \in S}  (W_i^{\top}A_k^*)^2  \right \} \right ]\\
&= \mathbb{E}_{S \in \mathbb{S}}\left [ \mathbf{1}_{i \in S} \left \{ m_2  \langle W_i, A_i^* \rangle^2 + m_2 \sum_{\substack{k \in S \\ k \neq i}}  \langle W_i, A_k^*\rangle^2  \right \} \right ] \\
&= \sum_{\{S \in \mathbb{S} : i \in S\}} m_2 \langle W_i, A_i^* \rangle^2 q_S +  \sum_{\substack{k=1 \\ k \neq i}}^h \sum_{\{S \in \mathbb{S} : i,k \in S\}} \langle W_i, A_k^*\rangle^2 q_S \\
&= m_2 \langle W_i, A_i^* \rangle^2 \sum_{\{S \in \mathbb{S} : i \in S\}} q_S + m_2 \sum_{\substack{k=1 \\ k \neq i}}^h \langle W_i, A_k^*\rangle^2 \left( \sum_{\{S \in \mathbb{S} : i,k \in S, i \neq k\}} q_S\right)  \\
&= q_i m_2 \langle W_i, A_i^* \rangle^2  + m_2 \sum_{\substack{k=1 \\ k \neq i}}^h q_{ik} \langle W_i, A_k^*\rangle^2 \\
&= h^{p-1} m_2 \langle W_i, A_i^* \rangle^2  + m_2 h^{2p-1} \textrm{ max } \langle W_i, A_k^*\rangle^2
\end{align*}

~\\
The above implies the following bounds,
\begin{align}\label{alpha_m2}
h^{p-1} m_2 (1 - h^{-p-\nu^2})^2 \leq \tilde{\alpha_i} \leq h^{p-1} m_2 (1 + h^{-p-\nu^2})^2 + m_2 h^{2p-1} (h^{-p-\nu^2} + h^{-\xi})^2
\end{align}
As long as $0< p < 2\xi$, $\tilde{\alpha_i} = \Theta \left( m_2 h^{p-1}\right)$

\paragraph{Bounding $\vert \vert \tilde{e_i} \vert \vert_2$}

\begin{align*}
\tilde{e_i} &= \mathbb{E}_{S \in \mathbb{S}}\left [ \mathbf{1}_{i \in S} \times \left \{    m_2\sum_{\substack{j, k \in S \\ j \neq i}} ( W_i ^\top A^*_k) (W_j^\top A^*_k)  W_j
+ (-2D)m_2 \sum_{\substack{j \in S \\ j \neq i}} ( W_i^{\top} A_j^*) A_j^* \right \} \right ]\\
&+ \mathbb{E}_{S \in \mathbb{S}}\left [ \mathbf{1}_{i \in S} \times \left \{ m_2 \sum_{\substack{j,k \in S \\ k \neq i}} (W_i^\top W_j) (W_j^\top A^*_k) A^*_k \right  \} \right ]\\
&= \mathbb{E}_{S \in \mathbb{S}}\left [ \mathbf{1}_{i \in S} \times m_2\left \{    \sum_{j (=k) \in S\setminus i } ( W_i^{\top} A_j^*)(  W_j^{\top}  A_j^*) W_j + \sum_{\substack{j \in S\setminus i\\ k \in S \setminus i,j}} ( W_i^{\top}  A_k^*)( W_j^{\top}  A_k^*) W_j + \sum_{\substack{j \in S\setminus i \\ k=i}} ( W_i^{\top}  A_i^*)(  W_j^{\top}  A_i^*) W_j\right \} \right ]\\
&+ \mathbb{E}_{S \in \mathbb{S}}\left [ \mathbf{1}_{i \in S} \times (-2D)m_2 \left \{ \sum_{\substack{j \in S \\ j \neq i}} ( W_i^{\top} A_j^*) A_j^*  \right \} \right ]\\
&+\mathbb{E}_{S \in \mathbb{S}}\left [ \mathbf{1}_{i \in S} \times m_2 \left \{ \sum_{\substack{k(=j) \in S \setminus i}} (W_i^\top W_k) (W_k^\top A^*_k) A^*_k  + \sum_{\substack{k \in S\setminus i\\ j \in S \setminus i,k}} (W_i^\top W_j) (W_j^\top A^*_k) A^*_k + \sum_{\substack{k \in S\setminus i \\ j=i}} (W_i^\top W_i) (W_i^\top A^*_k) A^*_k \right \} \right ]
\end{align*}

\begin{align*}
\tilde{e_i} &= m_2 \Bigg \{ \sum_{j=1, j \neq i}^h ( W_i^{\top} A_j^*)(  W_j^{\top}  A_j^*) W_j \sum_{\{S \in \mathbb{S}: i,j \in S, i \neq j\}} q_S + \sum_{\substack{j,k=1 \\j \neq k \neq i}}^h ( W_i^{\top}  A_k^*)( W_j^{\top}  A_k^*) W_j \sum_{\{S \in \mathbb{S}: i,j,k \in S, i \neq j \neq k\}} q_S \\
&+ \sum_{\substack{j =1 \\ j \neq i}}^h ( W_i^{\top}  A_i^*)(  W_j^{\top}  A_i^*) W_j \sum_{\{S \in \mathbb{S}: i,j \in S, i \neq j\}} q_S \Bigg \}\\
&+ (-2D)m_2 \left \{ \sum_{\substack{j =1 \\ j \neq i}}^h ( W_i^{\top} A_j^*) A_j^* \sum_{\{S \in \mathbb{S}: i,j \in S, i \neq j\}} q_S \right \} \\
&+ m_2 \Bigg \{ \sum_{\substack{k=1 \\ k \neq i}}^h (W_i^\top W_k) (W_k^\top A^*_k) A^*_k \sum_{\{S \in \mathbb{S}: i,k \in S, i \neq k\}} q_S  + \sum_{\substack{j,k =1\\ j \neq i \neq k}}^h (W_i^\top W_j) (W_j^\top A^*_k) A^*_k \sum_{\{S \in \mathbb{S}: i,j,k \in S, i \neq j \neq k\}} q_S \\
&+ \sum_{\substack{k =1 \\ k \neq i}}^h (W_i^\top W_i) (W_i^\top A^*_k) A^*_k \sum_{\{S \in \mathbb{S}: i,k \in S, i \neq k\}} q_S \Bigg \}\\
&= m_2 \Bigg \{ \sum_{j=1, j \neq i}^h q_{ij} ( W_i^{\top} A_j^*)(  W_j^{\top}  A_j^*) W_j + \sum_{\substack{j,k=1 \\j \neq k \neq i}}^h q_{ijk} ( W_i^{\top}  A_k^*)( W_j^{\top}  A_k^*) W_j \\
&+ \sum_{\substack{j =1 \\ j \neq i}}^h q_{ij} ( W_i^{\top}  A_i^*)(  W_j^{\top}  A_i^*) W_j \Bigg \} + (-2D)m_2 \left \{ \sum_{\substack{j =1 \\ j \neq i}}^h q_{ij} ( W_i^{\top} A_j^*) A_j^*  \right \} \\
&+ m_2 \Bigg \{ \sum_{\substack{k=1 \\ k \neq i}}^h q_{ik} (W_i^\top W_k) (W_k^\top A^*_k) A^*_k  + \sum_{\substack{j,k =1\\ j \neq i \neq k}}^h q_{ijk} (W_i^\top W_j) (W_j^\top A^*_k) A^*_k \\
&+ \sum_{\substack{k =1 \\ k \neq i}}^h q_{ik} (W_i^\top W_i) (W_i^\top A^*_k) A^*_k \Bigg \}
\end{align*}

\newpage 

~\\
Upper bounding the norm of this vector $\tilde{e}_i$ we get, 

\begin{align}
\label{ei_m2}
\nonumber ||\tilde{e_i}|| &\leq m_2 h^{2p-1} \left( \delta + \frac{\mu}{\sqrt{n}} \right) (1+\delta)^2 + m_2 h^{3p-1} \left( \delta + \frac{\mu}{\sqrt{n}} \right)^2 (1+\delta) \\
\nonumber &+ m_2 h^{2p-1} \left( \delta + \frac{\mu}{\sqrt{n}} \right) (1+\delta)^2 + 2D m_2 h^{2p-1} \left( \delta + \frac{\mu}{\sqrt{n}} \right) \\
\nonumber &+ m_2 h^{2p-1} \left( \delta^2 + 2\delta + \frac{\mu}{\sqrt{n}} \right) (1+\delta) + m_2 h^{3p-1} \left( \delta^2 + 2\delta + \frac{\mu}{\sqrt{n}} \right) \left(\delta + \frac{\mu}{\sqrt{n}} \right) \\
\nonumber &+ m_2 h^{2p-1} \left( \delta + \frac{\mu}{\sqrt{n}} \right) (1+\delta)^2\\
\nonumber &\leq m_2 h^{2p-1} (h^{-p-\nu^2} +2h^{-2p-2\nu^2} + h^{-3p-3\nu^2} + 2h^{-p-\nu^2 -\xi} + h^{-2p-2\nu^2-\xi} + h^{-\xi}) \\
\nonumber &+ m_2 h^{3p-1} (h^{-2p-2\nu^2} + h^{-3p-3\nu^2} + 2h^{-p-\nu^2 -\xi} + 2h^{-2p-2\nu^2-\xi} + h^{-2\xi} + h^{-p -\nu^2 -2\xi}) \\
\nonumber &+ m_2 h^{2p-1} (h^{-p-\nu^2} +2h^{-2p-2\nu^2} + h^{-3p-3\nu^2} + 2h^{-p-\nu^2 -\xi} + h^{-2p-2\nu^2-\xi} + h^{-\xi})\\
\nonumber &+ 2D m_2 h^{2p-1} (h^{-p-\nu^2} +h^{-\xi}) \\
\nonumber &+ m_2 h^{2p-1} (2h^{-p-\nu^2} +3h^{-2p-2\nu^2} + h^{-3p-3\nu^2} + h^{-p-\nu^2 -\xi} + h^{-\xi}) \\
\nonumber &+ m_2 h^{3p-1} (2h^{-2p-2\nu^2} + h^{-3p-3\nu^2} + 3h^{-p-\nu^2 -\xi} + h^{-2p-2\nu^2-\xi} + h^{-2\xi}) \\
&+ m_2 h^{2p-1} (h^{-p-\nu^2} +2h^{-2p-2\nu^2} + h^{-3p-3\nu^2} + 2h^{-p-\nu^2 -\xi} + h^{-2p-2\nu^2-\xi} + h^{-\xi})
\end{align}

~\\
If $D=1$ and $0<p<\xi$, we get $||\tilde{e_i}|| = o(m_2h^{p-1})$

\newpage 
\subsection { Estimating the $m_1$ dependent parts of the derivative}

We continue working in the same regime for the $W$ matrix as in the previous subsection. Hence the same inequalities as listed at the beginning of the previous subsection continue to hold and we use them  to get the following bounds, 

\paragraph {Bounding $\hat{\alpha_i}$}

\begin{align*}
\hat{\alpha_i} &= \mathbb{E}_{S \in \mathbb{S}} \Bigg[ \mathbf{1}_{i \in S} \times \Bigg \{ m_1^2 \sum_{\substack{k, l \in S \\ k \neq l}}  ( W_i ^\top A^*_k ) ( W_i ^\top A^*_l) -2m_1 \sum_{k \in S}  \epsilon_i  ( W_i^\top A^*_k) + \epsilon_i^2 \Bigg \} \Bigg] \\
&= \mathbb{E}_{S \in \mathbb{S}} \Bigg[ \mathbf{1}_{i \in S} \times \Bigg \{ m_1^2 \sum_{\substack{k \in S \\ k \neq i}} \langle W_i, A^*_k \rangle \langle W_i, A^*_i\rangle + m_1^2 \sum_{\substack{l \in S \\ l \neq i}} \langle W_i, A^*_i \rangle \langle W_i, A^*_l\rangle + m_1^2 \sum_{\substack{k, l \in S \\ k \neq l \\ k \neq i \\ l \neq i}}  \langle W_i, A^*_k \rangle \langle W_i, A^*_l \rangle \\
&- 2m_1 \epsilon_i \langle W_i, A^*_i \rangle - 2m_1 \sum_{\substack{k \in S \\ k \neq i}} \epsilon_i \langle W_i, A^*_k \rangle + \epsilon_i^2 \Bigg \} \Bigg] \\
&= 2m_1^2 \sum_{\substack{k = 1 \\ k \neq i}}^h \langle W_i, A^*_k \rangle \langle W_i, A^*_i\rangle \sum_{\{S \in \mathbb{S} : i,k \in S, k \neq i \}}q_S + m_1^2 \sum_{\substack{k, l = 1 \\ k \neq l \\ k \neq i \\ l \neq i}}^h  \langle W_i, A^*_k \rangle \langle W_i, A^*_l \rangle \sum_{\{S \in \mathbb{S} : i,k,l \in S, k \neq i \neq l \}}q_S \\
&- 2m_1 \epsilon_i \langle W_i, A^*_i \rangle \sum_{\{S \in \mathbb{S} : i \in S \}}q_S - 2m_1 \sum_{\substack{k = 1 \\ k \neq i}}^h \epsilon_i \langle W_i, A^*_k \rangle \sum_{\{S \in \mathbb{S} : i,k \in S, k \neq i \}}q_S + \epsilon_i^2 \sum_{\{S \in \mathbb{S} : i \in S \}}q_S\\
\implies \hat{\alpha_i} &= 2m_1^2 \sum_{\substack{k = 1 \\ k \neq i}}^h q_{ik} \langle W_i, A^*_k \rangle \langle W_i, A^*_i\rangle + m_1^2 \sum_{\substack{k, l = 1 \\ k \neq l \\ k \neq i \\ l \neq i}}^h q_{ikl} \langle W_i, A^*_k \rangle \langle W_i, A^*_l \rangle \\
&- 2m_1 q_i \epsilon_i \langle W_i, A^*_i \rangle - 2m_1 \sum_{\substack{k = 1 \\ k \neq i}}^h q_{ik} \epsilon_i \langle W_i, A^*_k \rangle + q_i \epsilon_i^2
\end{align*}

~\\
We plugin $\epsilon_i = 2m_1 h^p \left( \delta + \frac{\mu}{\sqrt{n}} \right)$ for $i = 1, \ldots, h$

\begin{align*}
|\hat{\alpha_i}| &\leq 2m_1^2 h^{2p-1} \left( \delta + \frac{\mu}{\sqrt{n}}\right) (1+\delta) + m_1^2 h^{3p-1} \left( \delta + \frac{\mu}{\sqrt{n}} \right)^2 + 4m_1^2 h^{2p-1} (1+\delta) \left( \delta + \frac{\mu}{\sqrt{n}} \right) \\
&+ 4m_1^2 h^{3p-1} \left( \delta + \frac{\mu}{\sqrt{n}} \right)^2 + 4m_1^2 h^{3p-1} \left( \delta + \frac{\mu}{\sqrt{n}} \right)^2\\
&= 2m_1^2 h^{2p-1} (h^{-p-\nu^2} + h^{-2p-2\nu^2} + h^{-p-\nu^2 -\xi} + h^{-\xi}) + m_1^2 h^{3p-1} (h^{-2p-2\nu^2} + 2h^{-p-\nu^2 -\xi} + h^{-2\xi}) \\
&+4m_1^2 h^{2p-1} (h^{-p-\nu^2} + h^{-2p-2\nu^2} + h^{-\xi} + h^{-p-\nu^2 -\xi}) + 4m_1^2 h^{3p-1} (h^{-2p-2\nu^2} + 2h^{-p-\nu^2 -\xi} + h^{-2\xi}) \\
&+ 4m_1^2 h^{3p-1} (h^{-2p-2\nu^2} + 2h^{-p-\nu^2 -\xi} + h^{-2\xi})
\end{align*}
~\\
This means that if $p < \xi$, $|\hat{\alpha_i}| = o( m_1^2 h^{p-1} )$. Putting  this together with the bounds obtained below \ref{alpha_m2}, we get that $\alpha_i = \Theta(m_2h^{p-1}) + o( m_1^2 h^{p-1} )$.

\paragraph{Bounding $\hat{\beta_i}$}

\begin{align*}
\hat{\beta_i} &= \mathbb{E}_{S \in \mathbb{S}} \Bigg[ \mathbf{1}_{i \in S} \times \Bigg \{  2 D m_1^2 \sum_{\substack{k \in S \\ k \neq i}} ( W_i^{\top} A_k^*) - D m_1 \epsilon_i + m_1 \sum_{j \in S} \epsilon_j (W_i^\top W_j) - m_1^2 \sum_{\substack{j,l \in S \\ l \neq i}} (W_i^\top W_j) (W_j^{\top} A_l^*) \Bigg \} \Bigg ] \\
&= 2D m_1^2 \sum_{\substack{k = 1 \\ k \neq i}}^h \langle W_i, A_k^* \rangle \sum_{\{S \in \mathbb{S} : i,k \in S, k \neq i \}}q_S -Dm_1 \epsilon_i \sum_{\{S \in \mathbb{S} : i \in S \}}q_S + m_1 \epsilon_i ||W_i||^2 \sum_{\{S \in \mathbb{S} : i \in S \}}q_S \\
&+m_1\sum_{j =1, j\neq i}^h \epsilon_j \langle W_i, W_j \rangle \sum_{\{S \in \mathbb{S} : i,j \in S, j \neq i \}}q_S - m_1^2 \sum_{\substack{l =1 \\ l \neq i}}^h ||W_i||^2 \langle W_i, A_l^* \rangle \sum_{\{S \in \mathbb{S} : i,l \in S, l \neq i \}}q_S \\
&-m_1^2 \sum_{\substack{l =1 \\ l \neq i}}^h \langle W_i, W_l \rangle \langle W_l, A_l^*\rangle \sum_{\{S \in \mathbb{S} : i,l \in S, l \neq i \}}q_S - m_1^2 \sum_{\substack{j,l =1 \\ l \neq i \\ j \neq l,i}}^h \langle W_i, W_j\rangle \langle W_j, A_l^* \rangle \sum_{\{S \in \mathbb{S} : i,j,l \in S, l \neq i \neq i \}}q_S \\
&= 2D m_1^2 \sum_{\substack{k = 1 \\ k \neq i}}^h q_{ik} \langle W_i, A_k^* \rangle -Dm_1 \epsilon_i q_i + m_1 \epsilon_i ||W_i||^2 q_i +m_1\sum_{j =1, j\neq i}^h \epsilon_j q_{ij} \langle W_i, W_j \rangle - m_1^2 \sum_{\substack{l =1 \\ l \neq i}}^h ||W_i||^2 \langle W_i, A_l^* \rangle q_{il} \\
&-m_1^2 \sum_{\substack{l =1 \\ l \neq i}}^h \langle W_i, W_l \rangle \langle W_l, A_l^*\rangle q_{il} - m_1^2 \sum_{\substack{j,l =1 \\ l \neq i \\ j \neq l,i}}^h \langle W_i, W_j\rangle \langle W_j, A_l^* \rangle q_{ijl}
\end{align*}
~\\
We plugin $\epsilon_i = 2m_1 h^p \left( \delta + \frac{\mu}{\sqrt{n}} \right)$ for $i = 1, \ldots, h$

\begin{align*}
|\hat{\beta_i}| &\leq 4Dm_1^2 h^{2p-1} \left(\delta + \frac{\mu}{\sqrt{n}} \right) + 2m_1^2 h^{2p-1} \left( \delta + \frac{\mu}{\sqrt{n}} \right) (1+\delta)^2 +  2m_1^2 h^{3p-1} \left( \delta + \frac{\mu}{\sqrt{n}} \right) \left( \delta^2 + 2\delta + \frac{\mu}{\sqrt{n}}\right)\\
&+ m_1^2 h^{2p-1}(1+\delta)^2\left( \delta + \frac{\mu}{\sqrt{n}} \right) + m_1^2 h^{2p-1} \left( \delta^2 + 2\delta + \frac{\mu}{\sqrt{n}}\right) (1+\delta) \\
&+ m_1^2 h^{3p-1} \left( \delta^2 + 2\delta + \frac{\mu}{\sqrt{n}}\right) \left(\delta + \frac{\mu}{\sqrt{n}}\right)\\
&= 4Dm_1^2 h^{2p-1} (h^{-p-\nu^2} + h^{-\xi} ) + 2m_1^2 h^{2p-1} (h^{-p -\nu^2} + 2h^{-2p -2\nu^2} + h^{-3p -3\nu^2} + h^{-\xi} + 2h^{-p -\nu^2-\xi} + h^{-2p -2\nu^2 - \xi})\\
&+  2m_1^2 h^{3p-1} (2h^{-2p -2\nu^2} + h^{-3p -3\nu^2} + 3h^{-p -\nu^2-\xi} + h^{-2p -2\nu^2 - \xi} + h^{-2\xi})\\
&+ m_1^2 h^{2p-1} (h^{-p -\nu^2} + 2h^{-2p -2\nu^2} + h^{-3p -3\nu^2} + h^{-\xi} + 2h^{-p -\nu^2-\xi} + h^{-2p -2\nu^2 - \xi})\\
&+ m_1^2 h^{2p-1}  (3h^{-2p -2\nu^2} + h^{-3p -3\nu^2} + h^{-p -\nu^2-\xi} + 2h^{-p -\nu^2} + h^{-\xi}) \\
&+ m_1^2 h^{3p-1}  (2h^{-2p -2\nu^2} + h^{-3p -3\nu^2} + 3h^{-p -\nu^2-\xi} + h^{-2p -2\nu^2 - \xi} + h^{-2\xi})
\end{align*}
~\\
This means that if $p < \xi$, $|\hat{\beta_i}| = o( m_1^2 h^{p-1} )$. Putting  this together with the bounds obtained below \ref{beta_m2}, we get that $\beta_i = \Theta(m_2h^{p-1}) + o( m_1^2 h^{p-1} )$.

\paragraph{Bounding $\vert \vert \hat{e_i} \vert \vert_2$}

\begin{align*}
\hat{e_i} &= \underbrace{\mathbb{E}_{S \in \mathbb{S}}\left [ \mathbf{1}_{i \in S} \times \left \{ \sum_{\substack{j \in S \\ j \neq i}} \epsilon_i \epsilon_j W_j
-m_1 \sum_{\substack{j, k \in S \\ j \neq i}}  ( W_j^\top A^*_k)  W_j \epsilon_i
-m_1 \sum_{\substack{j, k \in S \\ j \neq i}} \epsilon_j (W_i^\top A^*_k) W_j \right \}  \right ]}_{\hat{e_{i1}}} \\
&+ \underbrace{\mathbb{E}_{S \in \mathbb{S}}\left [ \mathbf{1}_{i \in S} \times \left \{
m_1^2 \sum_{\substack{j, k, l \in S \\ j \neq i \\ k \neq l}}  ( W_i ^\top A^*_k ) ( W_j ^\top A^*_l)  W_j\right \} \right ]}_{\hat{e_{i2}}}\\
&+ \underbrace{\mathbb{E}_{S \in \mathbb{S}}\left [ \mathbf{1}_{i \in S} \times \left \{ -2D m_1^2 \sum_{\substack{j, k \in S \\ j \neq i \\ k \neq i}} ( W_i^{\top} A_k^*) A_j^*
+ D m_1 \sum_{\substack{j \in S \\ j \neq i}} \epsilon_i A_j^*\right \} \right ]}_{\hat{e_{i3}}}\\
&+\underbrace{\mathbb{E}_{S \in \mathbb{S}}\left [ \mathbf{1}_{i \in S} \times \left \{ - m_1 \sum_{\substack{j,k \in S \\ k \neq i}} \epsilon_j  (W_i^\top W_j) A^*_k
+ m_1^2 \sum_{\substack{j,k,l \in S \\ k \neq i,l}} (W_i^\top W_j) (W_j^{\top} A_l^*)A^*_k\right \} \right ]}_{\hat{e_{i4}}}
\end{align*}

~\\
We estimate the different summands separately. 

\begin{align*}
\hat{e_{i1}} &= \mathbb{E}_{S \in \mathbb{S}}\left [ \mathbf{1}_{i \in S} \times \left \{ \sum_{\substack{j \in S \\ j \neq i}} \epsilon_i \epsilon_j W_j \right \}  \right ]\\
&+\mathbb{E}_{S \in \mathbb{S}}\left [ \mathbf{1}_{i \in S} \times (-m_1)\left \{ \sum_{\substack{j(=k) \in S \setminus  i}}  ( W_j^\top A^*_j)  W_j \epsilon_i + \sum_{\substack{j \in S \setminus i \\ k \in S \setminus i,j}}  ( W_j^\top A^*_k)  W_j \epsilon_i  + \sum_{\substack{j \in S \setminus i \\ k =i}}  ( W_j^\top A^*_i)  W_j \epsilon_i \right \}  \right ]\\
&+\mathbb{E}_{S \in \mathbb{S}}\left [ \mathbf{1}_{i \in S} \times (-m_1)\left \{ \sum_{\substack{j(=k) \in S \setminus i}} \epsilon_j (W_i^\top A^*_j) W_j + \sum_{\substack{ j \in S \setminus i \\ k \in S \setminus i,j}} \epsilon_j (W_i^\top A^*_k) W_j + \sum_{\substack{j \in S \setminus i \\ k=i }} \epsilon_j (W_i^\top A^*_i) W_j \right \} \right ]
\end{align*}

~\\
We substitute, $\epsilon = 2m_1h^p(h^{-p-\nu^2}+h^{-\xi})$ and for any two vectors $\x$ and $\y$ and any two scalars $a$ and $b$ we use the inequality, $\vert \vert a \x + b \y \vert \vert_2 \leq \vert a\vert_{max}\vert \vert \x \vert \vert _{2,max} + \vert b \vert _{max} \vert \vert \y \vert \vert _{2,max} $to get, 

~\\
\begin{align*}
\vert \vert \hat{e_{i1}} \vert \vert_2 &\leq 4m_1^2 h^{2p} \left( \delta+\frac{\mu}{\sqrt{n}} \right)^2 \sum_{j=1, j \neq i}^h q_{ij} ||W_j|| \\
&+ 2m_1^2 h^p \left( \delta+\frac{\mu}{\sqrt{n}} \right) \left( \sum_{j=1, j \neq i}^h q_{ij} \langle W_j, A^*_j\rangle W_j + \sum_{j,k=1, j \neq i, k \neq i,j}^h q_{ijk} \langle W_j, A^*_k\rangle W_j  + \sum_{j=1, j \neq i}^h q_{ij} \langle W_j, A^*_i\rangle W_j \right) \\
&+ 2m_1^2 h^p \left( \delta+\frac{\mu}{\sqrt{n}} \right) \left( \sum_{j=1, j \neq i}^h q_{ij} \langle W_i, A^*_j\rangle W_j + \sum_{j,k=1, j \neq i, k \neq i,j}^h q_{ijk} \langle W_i, A^*_k\rangle W_j  + \sum_{j=1, j \neq i}^h q_{ij} \langle W_i, A^*_i\rangle W_j \right)\\
\implies \vert \vert \hat{e_{i1}} \vert \vert_2 &\leq 4m_1^2 h^{2p} h^{2p-1}(1+\delta)\left( \delta+\frac{\mu}{\sqrt{n}} \right)^2 \\
&+ 2m_1^2 h^p \left( \delta+\frac{\mu}{\sqrt{n}} \right) \left( h^{2p-1}(1+\delta)^2 + h^{3p-1}\left(\delta +\frac{\mu}{\sqrt{n}} \right) (1+\delta)  + h^{2p-1}\left(\delta +\frac{\mu}{\sqrt{n}} \right) (1+\delta) \right) \\
&+ 2m_1^2 h^p \left( \delta+\frac{\mu}{\sqrt{n}} \right) \left( h^{2p-1}\left(\delta +\frac{\mu}{\sqrt{n}} \right) (1+\delta) + h^{3p-1}\left(\delta +\frac{\mu}{\sqrt{n}} \right) (1+\delta)  + h^{2p-1}(1+\delta)^2 \right)\\
\implies \vert \vert \hat{e_{i1}} \vert \vert_2 &\leq 4m_1^2 h^{4p-1}(1+\delta) \left( \delta+\frac{\mu}{\sqrt{n}} \right)^2 \\
&+ 2m_1^2 h^{3p-1} \left( \delta+\frac{\mu}{\sqrt{n}} \right) (1+\delta)^2 + 2m_1^2 h^{4p-1}\left(\delta +\frac{\mu}{\sqrt{n}} \right)^2 (1+\delta)  + 2m_1^2 h^{3p-1}\left(\delta +\frac{\mu}{\sqrt{n}} \right)^2 (1+\delta) \\
&+ 2m_1^2 h^{3p-1}\left(\delta +\frac{\mu}{\sqrt{n}} \right)^2 (1+\delta) + 2m_1^2 h^{4p-1}\left(\delta +\frac{\mu}{\sqrt{n}} \right)^2 (1+\delta)  + 2m_1^2 h^{3p-1}\left( \delta+\frac{\mu}{\sqrt{n}} \right)(1+\delta)^2\\
\implies \vert \vert \hat{e_{i1}} \vert \vert_2 &\leq 8m_1^2 h^{4p-1}(1+\delta) \left( \delta+\frac{\mu}{\sqrt{n}} \right)^2 + 4m_1^2 h^{3p-1} \left( \delta+\frac{\mu}{\sqrt{n}} \right) (1+\delta)^2 + 4m_1^2 h^{3p-1}\left(\delta +\frac{\mu}{\sqrt{n}} \right)^2 (1+\delta)
\implies \vert \vert \hat{e_{i1}} \vert \vert_2 &\leq 8m_1^2 h^{4p-1} (h^{-2p-2\nu^2} + h^{-3p-3\nu^2} + 2h^{-p-\nu^2 -\xi} + 2h^{-2p-2\nu^2 -\xi} + h^{-p-\nu^2 -2\xi} + h^{-2\xi}) \\
&+ 4m_1^2 h^{3p-1} (h^{-p-\nu^2} + h^{-3p-3\nu^2} + 2h^{-2p-2\nu^2} + h^{-\xi} + h^{-2p-2\nu^2 -\xi} + 2h^{-p-\nu^2 -\xi}) \\
&+ 4m_1^2 h^{3p-1}(h^{-2p-2\nu^2} + h^{-3p-3\nu^2} + 2h^{-p-\nu^2 -\xi} + 2h^{-2p-2\nu^2 -\xi} + h^{-p-\nu^2 -2\xi} + h^{-2\xi}) \\
&= 8m_1^2 h^{p-1} (h^{p-2\nu^2} + h^{-3\nu^2} + 2h^{p-\nu^2 + p-\xi} + 2h^{-2\nu^2 + p -\xi} + h^{-\nu^2 + 2p -2\xi} + h^{3p-2\xi}) \\
&+ 4m_1^2 h^{p-1} (h^{p-\nu^2} + h^{-p-3\nu^2} + 2h^{-2\nu^2} + h^{2p-\xi} + h^{-2\nu^2 -\xi} + 2h^{-\nu^2 +p-\xi}) \\
&+ 4m_1^2 h^{p-1}(h^{-2\nu^2} + h^{-p-3\nu^2} + 2h^{-\nu^2 + p -\xi} + 2h^{-2\nu^2 -\xi} + h^{-\nu^2 +p-2\xi} + h^{2p-2\xi})
\end{align*}
~\\
From the above it follows that, $\vert \vert \hat{e_{i1}} \vert \vert_2 = o(m_1^2h^{p-1})$ for $p < \nu^2$ and $2p < \xi$ .

~\\
\begin{align*}
\hat{e_{i2}} &= \mathbb{E}_{S \in \mathbb{S}}\left [ \mathbf{1}_{i \in S} \times m_1^2 \left \{ \sum_{\substack{j, k, l \in S \\ j \neq i \\ k \neq l}}  ( W_i ^\top A^*_k ) ( W_j ^\top A^*_l)  W_j\right \} \right ]\\
&= \mathbb{E}_{S \in \mathbb{S}}\Bigg [ \mathbf{1}_{i \in S} \times m_1^2 \Bigg \{ \sum_{\substack{j \in S \\ j \neq i}}  ( W_i ^\top A^*_j ) ( W_j ^\top A^*_i)  W_j + \sum_{\substack{j,k \in S \\ k \neq j \neq i}}  ( W_i ^\top A^*_k ) ( W_j ^\top A^*_i)  W_j + \sum_{\substack{j \in S \\ j \neq i}}  ( W_i ^\top A^*_i ) ( W_j ^\top A^*_j)  W_j \\
&+ \sum_{\substack{j,l \in S \\ l \neq j \neq i}}  ( W_i ^\top A^*_i ) ( W_j ^\top A^*_l)  W_j + \sum_{\substack{j,l \in S \\ l \neq j \neq i}}  ( W_i ^\top A^*_j ) ( W_j ^\top A^*_l)  W_j + \sum_{\substack{j,k \in S \\ k \neq j \neq i}}  ( W_i ^\top A^*_k ) ( W_j ^\top A^*_j)  W_j \\
&+ \sum_{\substack{j,k,l \in S \\ l \neq k \neq j \neq i}}  ( W_i ^\top A^*_k ) ( W_j ^\top A^*_l)  W_j \Bigg \} \Bigg ]\\
\implies \hat{e_{i2}} &= m_1^2 \Bigg \{ \sum_{\substack{j =1 \\ j \neq i}}^h q_{ij} ( W_i ^\top A^*_j ) ( W_j ^\top A^*_i)  W_j + \sum_{\substack{j,k =1 \\ k \neq j \neq i}}^h q_{ijk} ( W_i ^\top A^*_k ) ( W_j ^\top A^*_i)  W_j + \underbrace{\sum_{\substack{j =1 \\ j \neq i}}^h q_{ij} ( W_i ^\top A^*_i ) ( W_j ^\top A^*_j)  W_j}_{\mathbf{a}} \\
&+ \sum_{\substack{j,l =1 \\ l \neq j \neq i}}^h q_{ijl} ( W_i ^\top A^*_i ) ( W_j ^\top A^*_l)  W_j + \sum_{\substack{j,l =1 \\ l \neq j \neq i}}^h q_{ijl} ( W_i ^\top A^*_j ) ( W_j ^\top A^*_l)  W_j + \sum_{\substack{j,k =1 \\ k \neq j \neq i}}^h q_{ijk} ( W_i ^\top A^*_k ) ( W_j ^\top A^*_j)  W_j \\
&+ \sum_{\substack{j,k,l \in S \\ l \neq k \neq j \neq i}} q_{ijkl} ( W_i ^\top A^*_k ) ( W_j ^\top A^*_l)  W_j \Bigg \}\\
\implies ||\hat{e_{i2}}|| &\leq m_1^2 \Bigg \{ h^{2p-1} \left(\delta + \frac{\mu}{\sqrt{n}} \right)^2(1+\delta) + h^{3p-1} \left(\delta + \frac{\mu}{\sqrt{n}} \right)^2 (1+\delta) + ||\mathbf{a}|| \\
&+ h^{3p-1} \left(\delta + \frac{\mu}{\sqrt{n}} \right)(1+\delta)^2 + h^{3p-1} \left(\delta + \frac{\mu}{\sqrt{n}} \right)^2 (1+\delta) + h^{3p-1} \left(\delta + \frac{\mu}{\sqrt{n}} \right) (1+\delta)^2 \\
&+ h^{4p-1} \left(\delta + \frac{\mu}{\sqrt{n}} \right)^2 (1+\delta) \Bigg \}\\
\implies ||\hat{e_{i2}}|| &\leq m_1^2 \Bigg \{ h^{2p-1} (h^{-2p-2\nu^2} + h^{-3p-3\nu^2} + 2h^{-p-\nu^2 -\xi} + 2h^{-2p-2\nu^2 -\xi} + h^{-p-\nu^2 -2\xi} + h^{-2\xi}) \\
&+ h^{3p-1} (h^{-2p-2\nu^2} + h^{-3p-3\nu^2} + 2h^{-p-\nu^2 -\xi} + 2h^{-2p-2\nu^2 -\xi} + h^{-p-\nu^2 -2\xi} + h^{-2\xi}) \\
&+ ||\mathbf{a}||\\
&+ h^{3p-1} (h^{-p-\nu^2} + h^{-3p-3\nu^2} + 2h^{-2p-2\nu^2} + h^{-2p-2\nu^2 -\xi} + 2h^{-p-\nu^2 -\xi} + h^{-\xi})  \\
&+ h^{3p-1} (h^{-2p-2\nu^2} + h^{-3p-3\nu^2} + 2h^{-p-\nu^2 -\xi} + 2h^{-2p-2\nu^2 -\xi} + h^{-p-\nu^2 -2\xi} + h^{-2\xi}) \\
&+ h^{3p-1} (h^{-p-\nu^2} + h^{-3p-3\nu^2} + 2h^{-2p-2\nu^2} + h^{-2p-2\nu^2 -\xi} + 2h^{-p-\nu^2 -\xi} + h^{-\xi}) \\
&+ h^{4p-1} (h^{-2p-2\nu^2} + h^{-3p-3\nu^2} + 2h^{-p-\nu^2 -\xi} + 2h^{-2p-2\nu^2 -\xi} + h^{-p-\nu^2 -2\xi} + h^{-2\xi})  \Bigg \}
\end{align*}

\newpage 

\begin{align*}
\implies ||\hat{e_{i2}}|| &\leq m_1^2 \Bigg \{ h^{p-1} (h^{-p-2\nu^2} + h^{-2p-3\nu^2} + 2h^{-\nu^2 -\xi} + 2h^{-p-2\nu^2 -\xi} + h^{-\nu^2 -2\xi} + h^{p-2\xi}) \\
&+ h^{p-1} (h^{-2\nu^2} + h^{-p-3\nu^2} + 2h^{-\nu^2 + p -\xi} + 2h^{-2\nu^2 -\xi} + h^{-\nu^2 + p -2\xi} + h^{2p-2\xi}) \\
&+ ||\mathbf{a}||\\
&+ h^{p-1} (h^{p-\nu^2} + h^{-p-3\nu^2} + 2h^{-2\nu^2} + h^{-2\nu^2 -\xi} + 2h^{-\nu^2 + p -\xi} + h^{2p-\xi})  \\
&+ h^{p-1} (h^{-2\nu^2} + h^{-p-3\nu^2} + 2h^{-\nu^2 + p -\xi} + 2h^{-2\nu^2 -\xi} + h^{-\nu^2 + p -2\xi} + h^{2p-2\xi}) \\
&+ h^{p-1} (h^{p-\nu^2} + h^{-2p-3\nu^2} + 2h^{-2\nu^2} + h^{-2\nu^2 -\xi} + 2h^{-\nu^2 + p -\xi} + h^{2p-\xi}) \\
&+ h^{p-1} (h^{p-2\nu^2} + h^{-3\nu^2} + 2h^{p-\nu^2 + p -\xi} + 2h^{-2\nu^2 + p -\xi} + h^{-\nu^2 + 2p -2\xi} + h^{3p-2\xi})  \Bigg \}
\end{align*}

~\\
Now let us find a bound for $||\mathbf{a}||$.
\begin{align*}
\mathbf{a} &= \sum_{\substack{j =1 \\ j \neq i}}^h q_{ij} ( W_i ^\top A^*_i ) ( W_j ^\top A^*_j)  W_j \\
&= \langle W_i, A_i^* \rangle q_{ij} W_{-j}^\top \textrm{diag} (W_{-j} A^*_{-j}) 
\end{align*}
Where $A^*_{-j}$ is the dictionary $A^*$ with the $j$th column set to zero, $W_{-j}$ is the dictionary $W$ with the $j$th row set to zero, and $\textrm{diag} (W_{-j} A^*_{-j})$ is the $h$-dimensional vector containing the diagonal elements of the matrix $W_{-j} A^*_{-j}$. We also make use of the distributional assumption that $q_{ij}$ is the same for all $i,j$ in order to pull $q_{ij}$ out of the sum.
\begin{align*}
||\mathbf{a}||_2 &= h^{2p-2} \langle W_i, A_i^* \rangle \vert \vert W_{-j}^\top \textrm{diag} (W_{-j} A^*_{-j}) \vert \vert_2 \\
&\leq h^{2p-2}(1+\delta) ||W_{-j}^\top||_2 ||\textrm{diag} (W_{-j} A^*_{-j})||_2 \\
&\leq h^{2p-2}(1+\delta)^2 h^{1/2} \sqrt{\lambda_{\textrm{max}} (W^{\top}_{-j} W_{-j})} \\
&\leq h^{2p-2}(1+\delta)^2 h^{1/2} \sqrt{h \left( \delta ^2 + 2\delta + \frac{\mu}{\sqrt{n}} \right) + (1+\delta)^2 } \\
&= h^{p-1} \sqrt{h^{2p-2} \times h \times (1+\delta)^4 \times \left( h \left( \delta ^2 + 2\delta + \frac{\mu}{\sqrt{n}} \right) + (1+\delta)^2 \right)} \\
&= h^{p-1} \sqrt{h^{2p-1} \times (1+h^{-p-\nu^2})^4 \times \left( h (h^{-2p-2\nu^2} + 2h^{-p-\nu^2} + h^{-\xi} ) + (1+h^{-p-\nu^2})^2 \right)} \\
&= h^{p-1} \sqrt{(1+h^{-p-\nu^2})^4 \times ( h^{-2\nu^2} + 2h^{p-\nu^2} + h^{2p-\xi} + h^{2p-1}(1+h^{-p-\nu^2})^2)}
\end{align*}
Here $||W_{-j}^\top||_2$ is the spectral norm of $W_{-j}^\top$, and is the top singular value of the matrix. We use Gershgorin's Circle theorem to bound the top eigenvalue of $W^{\top}_{-j}W_{-j}$ by its maximum row sum.

~\\
If $p < \frac{\xi}{2}$, $p < \frac{1}{2}$, and $p < \nu^2$, then $||\hat{e_{i2}}|| = o(m_1^2 h^{p-1})$

\newpage
\begin{align*}
\hat{e_{i3}} &= \mathbb{E}_{S \in \mathbb{S}}\left [ \mathbf{1}_{i \in S} \times \left \{ D m_1 \sum_{\substack{j \in S \\ j \neq i}} \epsilon_i A_j^* - 2D m_1^2 \sum_{\substack{j, k \in S \\ j \neq i \\ k \neq i}} ( W_i^{\top} A_k^*) A_j^* \right \} \right ] \\
&= \mathbb{E}_{S \in \mathbb{S}}\left [ \mathbf{1}_{i \in S} \times \left \{ D m_1 \sum_{\substack{j \in S \\ j \neq i}} \epsilon_i A_j^* - 2D m_1^2 \sum_{\substack{j \in S \\ j \neq i}} ( W_i^{\top} A_j^*) A_j^* -2D m_1^2 \sum_{\substack{j, k \in S \\ k \neq j \neq i}} ( W_i^{\top} A_k^*) A_j^* \right \} \right ] \\
&= D m_1 \sum_{\substack{j =1 \\ j \neq i}}^h \epsilon_i A_j^* \sum_{\{S \in \mathbb{S}: i,j \in S, i\neq j \} } q_S - 2D m_1^2 \sum_{\substack{j =1 \\ j \neq i}}^h ( W_i^{\top} A_j^*) A_j^* \sum_{\{S \in \mathbb{S}: i,j \in S, i\neq j \} } q_S \\
&- 2D m_1^2 \sum_{\substack{j,k =1 \\ k \neq j \neq i}}^h ( W_i^{\top} A_k^*) A_j^* \sum_{\{S \in \mathbb{S}: i,j,k \in S, i\neq j \neq k \} } q_S \\
&= D m_1 \sum_{\substack{j =1 \\ j \neq i}}^h q_{ij} \epsilon_i A_j^* - 2D m_1^2 \sum_{\substack{j =1 \\ j \neq i}}^h q_{ij} ( W_i^{\top} A_j^*) A_j^* - 2D m_1^2 \sum_{\substack{j,k =1 \\ k \neq j \neq i}}^h q_{ijk} ( W_i^{\top} A_k^*) A_j^*
\end{align*}

~\\
We plugin $\epsilon_i = 2m_1 h^p \left( \delta + \frac{\mu}{\sqrt{n}} \right)$ for $i = 1, \ldots, h$

\begin{align*}
||\hat{e_{i3}}|| &\leq 2Dm_1^2 h^{3p-1} \left( \delta + \frac{\mu}{\sqrt{n}} \right) + 2Dm_1^2 h^{2p-1} \left( \delta + \frac{\mu}{\sqrt{n}} \right) + 2D m_1^2 h^{3p-1} \left( \delta + \frac{\mu}{\sqrt{n}} \right) \\
&= 4Dm_1^2 h^{3p-1} (h^{-p-\nu^2} + h^{-\xi}) + 2Dm_1^2 h^{2p-1} (h^{-p-\nu^2} + h^{-\xi}) \\
&= 4Dm_1^2 h^{p-1} (h^{p-\nu^2} + h^{2p-\xi}) + 2Dm_1^2 h^{p-1} (h^{-\nu^2} + h^{p-\xi})
\end{align*}

~\\
This means for $D=1$, $p < \nu^2$ and $p < \frac{\xi}{2}$, we have $||\hat{e_{i3}}|| = o(m_1^2 h^{p-1})$



\newpage 
~\\
\begin{align*}
\hat{e_{i4}} &= \mathbb{E}_{S \in \mathbb{S}}\left [ \mathbf{1}_{i \in S} \times \left \{ - m_1 \sum_{\substack{j,k \in S \\ k \neq i}} \epsilon_j  (W_i^\top W_j) A^*_k
+ m_1^2 \sum_{\substack{j,k,l \in S \\ k \neq i,l}} (W_i^\top W_j) (W_j^{\top} A_l^*)A^*_k\right \} \right ]\\
&= \mathbb{E}_{S \in \mathbb{S}}\left [ \mathbf{1}_{i \in S} \times (-m_1) \left \{  \sum_{k(=j) \in S \setminus i} \epsilon_k  (W_i^\top W_k) A^*_k + \sum_{\substack{j \in S \setminus i \\ k \in S \setminus i,j}} \epsilon_j  (W_i^\top W_j) A^*_k + \sum_{\substack{k \in S \setminus i \\ j = i}} \epsilon_j  (W_i^\top W_i) A^*_k \right \} \right ]\\
&+ \mathbb{E}_{S \in \mathbb{S}}\left [ \mathbf{1}_{i \in S} \times m_1^2 \left \{  \sum_{\substack{j,k,l \in S \\ k \neq i,l}} (W_i^\top W_j) (W_j^{\top} A_l^*)A^*_k\right \} \right ]\\
&=\mathbb{E}_{S \in \mathbb{S}}\left [ \mathbf{1}_{i \in S} \times (-m_1) \left \{  \sum_{k(=j) \in S \setminus i} \epsilon_k  (W_i^\top W_k) A^*_k + \sum_{\substack{j \in S \setminus i \\ k \in S \setminus i,j}} \epsilon_j  (W_i^\top W_j) A^*_k + \sum_{\substack{k \in S \setminus i \\ j = i}} \epsilon_j  (W_i^\top W_i) A^*_k \right \} \right ]\\
&+ \mathbb{E}_{S \in \mathbb{S}}\Bigg [ \mathbf{1}_{i \in S} \times m_1^2 \Bigg \{  \sum_{\substack{k \in S \\ k \neq i}} (W_i^\top W_i) (W_i^{\top} A_i^*)A^*_k  + \sum_{\substack{k \in S \\ k \neq i}} (W_i^\top W_k) (W_k^{\top} A_i^*)A^*_k + \sum_{\substack{j,k \in S \\ j \neq k \neq i}} (W_i^\top W_j) (W_j^{\top} A_i^*)A^*_k \\
&+ \sum_{\substack{k,l \in S \\ \ \neq k \neq i}} (W_i^\top W_i) (W_i^{\top} A_l^*)A^*_k + \sum_{\substack{k,l \in S \\ l \neq k \neq i}} (W_i^\top W_k) (W_k^{\top} A_l^*)A^*_k + \sum_{\substack{k,l \in S \\ l \neq k \neq i}} (W_i^\top W_l) (W_l^{\top} A_l^*)A^*_k \\
&+ \sum_{\substack{j,k,l \in S \\ j \neq k \neq l \neq i}} (W_i^\top W_j) (W_j^{\top} A_l^*)A^*_k
\Bigg \} \Bigg ]\\
\hat{e_{i4}} &= (-m_1) \left \{  \sum_{k=1, k \neq i}^h q_{ik} \epsilon_k  (W_i^\top W_k) A^*_k + \sum_{\substack{j,k =1 \\ j \neq k \neq i}}^h q_{ijk} \epsilon_j  (W_i^\top W_j) A^*_k + \sum_{\substack{k =1 \\ k \neq i}}^h q_{ik} \epsilon_i  (W_i^\top W_i) A^*_k \right \}\\
&+ m_1^2 \Bigg \{  \underbrace{\sum_{\substack{k =1 \\ k \neq i}}^h q_{ik} (W_i^\top W_i) (W_i^{\top} A_i^*)A^*_k}_{\mathbf{b}} + \sum_{\substack{k =1 \\ k \neq i}}^h q_{ik} (W_i^\top W_k) (W_k^{\top} A_i^*)A^*_k + \sum_{\substack{j,k =1 \\ j \neq k \neq i}}^h q_{ijk} (W_i^\top W_j) (W_j^{\top} A_i^*)A^*_k \\
&+ \sum_{\substack{k,l =1 \\ l \neq k \neq i}}^h q_{ikl} (W_i^\top W_i) (W_i^{\top} A_l^*)A^*_k + \sum_{\substack{k,l =1 \\ l \neq k \neq i}}^h q_{ikl} (W_i^\top W_k) (W_k^{\top} A_l^*)A^*_k + \sum_{\substack{k,l =1 \\ l \neq k \neq i}}^h q_{ikl} (W_i^\top W_l) (W_l^{\top} A_l^*)A^*_k \\
&+ \sum_{\substack{j,k,l =1 \\ j \neq k \neq l \neq i}}^h q_{ijkl} (W_i^\top W_j) (W_j^{\top} A_l^*)A^*_k
\Bigg \}
\end{align*}

\newpage 
~\\
We plugin $\epsilon_i = 2m_1 h^p \left( \delta + \frac{\mu}{\sqrt{n}} \right)$ for $i = 1, \ldots, h$ in the above to get,

\begin{align*}
||\hat{e_{i4}}|| &\leq 2m_1^2 h^{3p-1} \left( \delta + \frac{\mu}{\sqrt{n}} \right)^2 + 2m_1^2 h^{4p-1} \left( \delta + \frac{\mu}{\sqrt{n}} \right) \left( \delta^2 + 2\delta + \frac{\mu}{\sqrt{n}} \right) + 2m_1^2 h^{3p-1} \left( \delta + \frac{\mu}{\sqrt{n}} \right) (1+\delta)^2 \\
&+m_1^2||\mathbf{b}|| + m_1^2 h^{2p-1} \left( \delta + \frac{\mu}{\sqrt{n}} \right) \left( \delta^2 + 2\delta + \frac{\mu}{\sqrt{n}} \right) + m_1^2 h^{3p-1} \left( \delta + \frac{\mu}{\sqrt{n}} \right) \left( \delta^2 + 2\delta + \frac{\mu}{\sqrt{n}} \right) \\
&+ m_1^2 h^{3p-1} (1+\delta)^2 \left( \delta + \frac{\mu}{\sqrt{n}} \right) + m_1^2 h^{3p-1} \left( \delta + \frac{\mu}{\sqrt{n}} \right) \left( \delta^2 + 2\delta + \frac{\mu}{\sqrt{n}} \right) + m_1^2 h^{3p-1} (1+\delta) \left( \delta^2 + 2\delta + \frac{\mu}{\sqrt{n}} \right) \\
&+ m_1^2 h^{4p-1} \left( \delta + \frac{\mu}{\sqrt{n}} \right) \left( \delta^2 + 2\delta + \frac{\mu}{\sqrt{n}} \right)\\
\implies ||\hat{e_{i4}}|| &\leq 2m_1^2 h^{3p-1} \left( \delta + \frac{\mu}{\sqrt{n}} \right)^2 + 3m_1^2 h^{4p-1} \left( \delta + \frac{\mu}{\sqrt{n}} \right) \left( \delta^2 + 2\delta + \frac{\mu}{\sqrt{n}} \right) + 3m_1^2 h^{3p-1} \left( \delta + \frac{\mu}{\sqrt{n}} \right) (1+\delta)^2 \\
&+m_1^2||\mathbf{b}|| + m_1^2 h^{2p-1} \left( \delta + \frac{\mu}{\sqrt{n}} \right) \left( \delta^2 + 2\delta + \frac{\mu}{\sqrt{n}} \right) +2 m_1^2 h^{3p-1} \left( \delta + \frac{\mu}{\sqrt{n}} \right) \left( \delta^2 + 2\delta + \frac{\mu}{\sqrt{n}} \right) \\
&  + m_1^2 h^{3p-1} (1+\delta) \left( \delta^2 + 2\delta + \frac{\mu}{\sqrt{n}} \right)\\
\implies ||\hat{e_{i4}}|| &\leq 2m_1^2 h^{3p-1} ( h^{-2p-2\nu^2} + 2h^{-p-\nu^2 - \xi} + h^{-2\xi}) \\
&+ 3m_1^2 h^{4p-1} (h^{-3p-3\nu^2} + 2h^{-2p-2\nu^2} + 3h^{-p-\nu^2 - \xi} + h^{-2p-2\nu^2 -\xi} + h^{-2\xi}) \\
&+ 3m_1^2 h^{3p-1} (h^{-3p-3\nu^2} + 2h^{-2p-2\nu^2} + 2h^{-p-\nu^2 - \xi} + h^{-2p-2\nu^2 -\xi} + h^{-\xi} + h^{-p-\nu^2}) \\
&+ m_1^2 ||\mathbf{b}|| \\
&+ m_1^2 h^{2p-1} (h^{-3p-3\nu^2} + 2h^{-2p-2\nu^2} + 3h^{-p-\nu^2 - \xi} + h^{-2p-2\nu^2 -\xi} + h^{-2\xi}) \\
&+2 m_1^2 h^{3p-1} (h^{-3p-3\nu^2} + 2h^{-2p-2\nu^2} + 3h^{-p-\nu^2 - \xi} + h^{-2p-2\nu^2 -\xi} + h^{-2\xi}) \\
&  + m_1^2 h^{3p-1} (h^{-3p-3\nu^2} + 3h^{-2p-2\nu^2} + h^{-p-\nu^2 - \xi} + h^{-\xi} + 2h^{-p-\nu^2})\\
\implies ||\hat{e_{i4}}|| &\leq 2m_1^2 h^{p-1} ( h^{-2\nu^2} + 2h^{-\nu^2 + p- \xi} + h^{2p-2\xi}) \\
&+ 3m_1^2 h^{p-1} (h^{-3\nu^2} + 2h^{-p-2\nu^2} + 3h^{p-\nu^2 + p- \xi} + h^{-2\nu^2 + p-\xi} + h^{3p-2\xi}) \\
&+ 3m_1^2 h^{p-1} (h^{-p-3\nu^2} + 2h^{-2\nu^2} + 2h^{-\nu^2 + p - \xi} + h^{-2\nu^2 -\xi} + h^{2p-\xi} + h^{p-\nu^2}) \\
&+ m_1^2 ||\mathbf{b}|| \\
&+ m_1^2 h^{p-1} (h^{-2p-3\nu^2} + 2h^{-p-2\nu^2} + 3h^{-\nu^2 - \xi} + h^{-p-2\nu^2 -\xi} + h^{p-2\xi}) \\
&+2 m_1^2 h^{p-1} (h^{-p-3\nu^2} + 2h^{-2\nu^2} + 3h^{-\nu^2 + p - \xi} + h^{-2\nu^2 -\xi} + h^{2p-2\xi}) \\
&  + m_1^2 h^{p-1} (h^{-p-3\nu^2} + 3h^{-2\nu^2} + h^{-\nu^2 + p - \xi} + h^{2p-\xi} + 2h^{p-\nu^2})
\end{align*}

~\\
Now let us find a bound for $||\mathbf{b}||$.
\begin{align*}
\mathbf{b} &= \sum_{\substack{k =1 \\ k \neq i}}^h q_{ik} (W_i^\top W_i) (W_i^{\top} A_i^*)A^*_k \\
&= \langle W_i, W_i \rangle \langle W_i, A_i^* \rangle q_{ik} A^*_{-i} \mathbf{1}_h 
\end{align*}
Where $A^*_{-i}$ is the dictionary $A^*$ with the $i$th column set to zero, and $\mathbf{1}_h \in \mathbb{R}^h$ is the $h$-dimensional vector of all ones. Here we make use of the distributional assumption that $q_{ik}$ is the same for all $i,k$ in order to pull $q_{ik}$ out of the sum.
\begin{align*}
||\mathbf{b}||_2 &= h^{2p-2} \langle W_i, W_i \rangle \langle W_i, A_i^* \rangle \vert \vert A^*_{-i} \mathbf{1}_h \vert \vert_2 \\
&\leq h^{2p-2}(1+\delta)^3 ||A^*_{-i}||_2 ||\mathbf{1}_h||_2 \\
&= h^{2p-2}(1+\delta)^3 h^{1/2} \sqrt{\lambda_{\textrm{max}} (A^{*\top}_{-i} A^{*}_{-i})} \\
&= h^{2p-2}(1+\delta)^3 h^{1/2} \sqrt{h\frac{\mu}{\sqrt{n}} + 1} \\
&= h^{p-1} \sqrt{h^{2p-2} \times h \times (1+\delta)^6 \times \left( h\frac{\mu}{\sqrt{n}} +1 \right)} \\
&= h^{p-1} \sqrt{h^{2p-1} \times (1+h^{-p-\nu^2})^6 \times \left( h^{1-\xi} +1 \right)} \\
&= h^{p-1} \sqrt{(1+h^{-p-\nu^2})^6 \times ( h^{2p-\xi} + h^{2p-1})}
\end{align*}
Here $||A^*_{-i}||_2$ is the spectral norm of $A^*_{-i}$, and is the top singular value of the matrix. We use Gershgorin's Circle theorem to bound the top eigenvalue of $A^{*\top}_{-i}A^*_{-i}$ by its maximum row sum.
~\\
If $p < \frac{\xi}{2}$, $p < \frac{1}{2}$, and $p < \nu^2$, then $||\hat{e_{i4}}|| = o(m_1^2 h^{p-1})$. Now we combine the above obtained bounds for $\Vert \hat{e_{it}}\Vert$ (for $t \in \{1,2,3,4\}$) with the bound obtained below equation \ref{ei_m2} to say that, $\Vert e_i \Vert = o(\max\{m_1^2,m_2\}h^{p-1})$

\subsection {About $\alpha_i - \beta_i$}

Remembering that $D=1$ and doing a close scrutiny of the terms in \ref{alpha_m2} and \ref{beta_m2} will indicate that the coefficients are the \emph {same} for the $m_2h^{p-1}$ term in each of them. (which is the term with the highest $h$ scaling in the $m_2$ dependent parts of $\alpha_i$ and $\beta_i$). So this largest term cancels off in the difference and we are left with the sub-leading order terms coming from both their $m_1^2$ as well as the $m_2$ parts and this gives us,

\[ \alpha_i - \beta_i = o(\max\{m_1^2,m_2\}h^{p-1}) \]

\end{document}